\documentclass[sigconf]{aamas} 

\usepackage{soul}
\usepackage{url}
\usepackage{graphicx}

\usepackage{balance}

\usepackage{soul}
\usepackage{url}
\hypersetup{hidelinks}
\usepackage{graphicx}
\usepackage{amsmath}
\usepackage{booktabs}
\usepackage{setspace}

\PassOptionsToPackage{dvipsnames,table,prologue}{xcolor}

\usepackage[utf8]{inputenc} 
\usepackage[T1]{fontenc}    
\usepackage{hyperref}       
\usepackage{url}            
\usepackage{booktabs}       
\usepackage{amsfonts}       
\usepackage{amsmath}
\usepackage{mathtools}
\usepackage{bbm}
\usepackage{amsthm}
\usepackage{nicefrac}       
\usepackage{microtype}      
\usepackage{algorithm}
\usepackage[noend]{algorithmic}
\usepackage[font+=smaller,skip=0pt]{subcaption}
\usepackage{color}
\usepackage{wrapfig}
\usepackage{paralist}
\usepackage{epigraph}

\allowdisplaybreaks
\DeclareMathOperator*{\argmax}{\arg\!\max}

\newtheorem{theorem}{Theorem}
\newtheorem{lemma}{Lemma}
\newenvironment{hproof}{%
  \renewcommand{\proofname}{Proof Sketch}\proof}{\endproof}
\newcommand{\NaiveGroupFair}{\textsc{NaiveGroupFair}}
\newcommand{\defeq}{\vcentcolon=}

\setcopyright{ifaamas}
\acmConference[AAMAS '22]{Proc.\@ of the 21st International Conference
on Autonomous Agents and Multiagent Systems (AAMAS 2022)}{May 9--13, 2022}
{Online}{P.~Faliszewski, V.~Mascardi, C.~Pelachaud,
M.E.~Taylor (eds.)}
\copyrightyear{2022}
\acmYear{2022}
\acmDOI{}
\acmPrice{}
\acmISBN{}

\title[Group Fairness in Bandits with Biased Feedback]{Group Fairness in Bandits with Biased Feedback}

\author{Candice Schumann}
\affiliation{
  \institution{University of Maryland}
  \city{College Park}
  \country{USA}%
}
\email{candiceschumann@gmail.com}

\author{Zhi Lang}
\affiliation{
  \institution{University of Maryland}
  \city{College Park}
  \country{USA}%
}
\email{willylang1996@gmail.com}

\author{Nicholas Mattei}
\affiliation{
  \institution{Tulane University}
  \city{New Orleans}
  \country{USA}%
}
\email{nsmattei@gmail.com}

\author{John P.~Dickerson}
\affiliation{
  \institution{University of Maryland}
  \city{College Park}
  \country{USA}%
}
\email{johnd@umd.edu}

\begin{abstract}
  We propose a novel formulation of group fairness with biased feedback in the contextual multi-armed bandit (CMAB) setting.  In the CMAB setting, a sequential decision maker must, at each time step, choose an arm to pull from a finite set of arms after observing some context for each of the potential arm pulls.  In our model, arms are partitioned into two or more sensitive groups based on some  protected feature(s) (e.g., age, race, or socio-economic status). Initial rewards received from pulling an arm may be distorted due to some unknown societal or measurement bias. We assume that in reality these groups are equal despite the biased  feedback received by the agent. To alleviate this, we learn a societal bias term which can be used to both find the source of bias and to potentially fix the problem outside of the algorithm. We provide a novel algorithm that can accommodate this notion of fairness for an arbitrary number of groups, and provide a theoretical bound on the regret for our algorithm. We validate our algorithm using synthetic data and two real-world datasets for intervention settings wherein we want to allocate resources fairly across groups.

\end{abstract}


\begin{document}
\maketitle

\epigraph{{\it Knowing that one may be subject to bias is one thing; being able to correct it is another.}}{{Jon Elster}}

\section{Introduction}
In many online settings, a computational or human agent must sequentially select an item from a slate, receive feedback on that selection, and then use that feedback to learn how to select the best items in the following rounds.  Within computer science, economics, and operations research circles, this is typically modeled as a \emph{multi-armed bandit (MAB)} problem \cite{SuBa17a}.  Examples include algorithms for selecting what advertisements to display to users on a webpage~\cite{Mary:15:BanditRecommender}, systems for dynamic pricing~\cite{Misra19:Dynamic}, and content recommendation services~\cite{Li:10:PersonalRec}.  Indeed, such decision-making systems continue to expand in scope, making ever more important decisions in our lives such as setting bail \cite{Corbett18MeasureFairness}, making hiring decisions~\cite{Bogen19:All,Schumann:2019:CohortSelection}, and policing~\cite{Rudin13:Predictive}.  Thus, the study of the properties of these algorithms is of paramount importance as highlighted by \citet{Chouldechova:18:FrontiersFairness} motivating priorities for fairness research in machine learning.

In the basic MAB setting, there are $n$ \emph{arms}, each associated with a fixed but unknown reward probability distribution \cite{Lai:85:EfficientAllocation,Auer:02:UCB}. At each time step $t \in T$, an agent pulls an arm and receives a reward that is independent of any previous action and follows the selected arm's probability distribution.  
The goal of the agent is to maximize the total collected reward over time. A generalization of MAB is the contextual multi-armed bandit (CMAB) where the agent observes a $d$-dimensional \emph{context} along with the observed rewards to choose a new arm. In the CMAB problem, the agent learns the relationship between contexts and rewards and selects the best arm~\cite{Agrawal:13:ThompsonSampling}.

Yet, the use of MAB- and CMAB-based systems often results in behavior that is societally repugnant. \citet{Sweeney13:Discrimination} noted that queries for public records on Google resulted in different contextual advertisements based on whether the query target had a traditionally African American or Caucasian name; in the former case advertisements were more likely to contain text relating to criminal incidents.  Following that initial report similar instances continue to be observed, both in the bandit setting and in the general machine learning world~\cite{ONeil16:Weapons}.  In lockstep, the academic community has begun developing approaches to tackling issues of (un)fairness in learning settings. We have an opportunity to identify and understand why the data we have may be \emph{causing} the bias. 


A Computing Community Consortium (CCC) report on fairness in ML identified that most studies of fairness are focused on classification problems~\cite{Chouldechova:18:FrontiersFairness}.  These works define a statistical notion of fairness, typically a notion of equal treatment of equals \cite{Rawls71a}, and propose algorithms to abide by these constraints.  Two issues identified by \citet{Chouldechova:18:FrontiersFairness} that we address in this paper are extensions to notions of \emph{group fairness} and looking at fairness in \emph{online dynamic systems}, e.g., CMABs. We address these gaps by formalizing and providing an algorithm for fairness with biased feedback when the arms of the bandit can be partitioned into groups. Direct applications of our work including scenarios discussed within the AAMAS community like aiding the allocation of human resources in talent sourcing~\cite{Schumann20:Fairness}.

The recent AI100 study \cite{ai1002021}, whose goal is to take a broad and long term look at the opportunities and pitfalls for AI researchers, has highlighted the need to develop systems that work \emph{with} humans, providing oversight, transparency, and explanation. Our bandit formulation is one step towards creating more human-centered AI \cite{shneiderman2020human}, a new area of study that seeks to understand and balance computer automation and autonomy with the level of human control in a given system. Many of the negative applications of MAB based systems we have discussed so far too often occur because there is too much autonomy given to the system, and it optimizes away from what humans or society considers desirable. By explicitly modeling the underlying bias term, we hope to improve computer aided decision making by understanding and mitigating the dangers that can occur when there are excessive levels of human control or excessive levels of computer autonomy; leading to systems that are more transparent, auditable, and trustworthy.

\textbf{Running Example.}  As a running example throughout the paper, imagine the position of an agent at a bank or a lender on a micro-lending site.  Here, the agent must sequentially pick loans to fund.  In many cases, such as the micro-lending site Kiva, 
a user is presented with a slate of potential loans they may fund when they log in and this slate is generated by a recommender system \cite{Burke:RecFair2020}.  Each of these loans, i.e. arms, has a context which includes attributes of the applicant (e.g., personal statement, repayment history, business plan).  The loans can also be  partitioned into sets of $m$ sensitive attributes, e.g. location, race, or gender.  In the simplest case, assume we have two female applicants and two male applicants on the slate at a given time.  
We also assume that when pulling an arm from, for example, a female applicant, there is some societal bias introduced into the reward. Yet, in many settings (and, as we assume in this work), the average \emph{true} (i.e., unbiased) reward across groups is equal.  We want to balance the number of times the agent selects women versus men given this societal bias built into the feedback.

While we use loans as our running example, our notion of regret could be extended to a number of other areas including recent work in MAB problems on hiring situations \cite{Schumann:2019:TieredInterviewing}, including the recent AAMAS Blue Sky Paper by \citet{Schumann20:Fairness} specifically calling for the community to contribute to fair hiring. One could imagine a situation where hiring decisions are made w.r.t.\ a short-term reward signal that is biased,\footnote{Class-based bias presents itself within seconds of an in-person interview; see https://news.yale.edu/2019/10/21/yale-study-shows-class-bias-hiring-based-few-seconds-speech.} versus a longer-term reward of performance which is less biased, e.g., via an end-of-year review that is based on a more quantitative metric such as on-the-job performance. A similar argument can be made about school admissions or matching workers to online tasks in a crowdwork setting.



\paragraph{Contributions.}
We propose a novel formulation of group fairness in the contextual multi-armed bandit (CMAB) setting.  In our model, arms are partitioned into two or more sensitive groups based on some protected feature, e.g., race. Despite the fact that there may be differences in expected payout between the groups, we may wish to ensure some form of fairness between picking arms from the various groups.  Our goal is to capture the phenomena where we want to balance the arms being pulled from both groups and (learn to) ignore societal bias generated by sensitive group membership.  We define two novel notions of reward and regret to capture implicit societal bias: proportional parity and equal group parity.  We provide a novel algorithm that can accommodate these notions of fairness for an arbitrary number of groups, learn the societal bias term itself, and provide bounds on the regret for our algorithm.  We validate our algorithms using synthetic data and real-world datasets for intervention settings wherein we want to allocate resources fairly across protected groups.\footnote{A full version of this paper, complete with an appendix containing proofs and additional experiments, can be found at \url{https://arxiv.org/abs/1912.03802}.  We will also periodically update this work should typos or other errors be found; if you see any, please feel free to reach out!  Code to reproduce the experiments is also available at \url{https://github.com/candiceschumann/groupfairtreatment}.}





\section{Related Work}



Fairness in machine learning has become one of the most active topics in computer science \cite{Chouldechova:18:FrontiersFairness}.  The idea of using formal notions of fairness, i.e. axioms or properties, to design decision schemes has a long history in economics and political economy \cite{Rawls71a,young1995equity}.
Typically within ML research, fairness is operationalized using 
the Rawlsian idea that similar individuals should be treated similarly; formally extended to the classification setting by \citet{Dwork12:Fairness}, who provided algorithms to ensure individual fairness at the cost of the utility of the overall system.  Their work underscores that in many cases statistical parity is not sufficient to ensure individual fairness, as we may treat groups fairly but in doing so may be very unfair to some specific individual.
Determining when, how, and if to define fairness is an ongoing discussion with roots well before the time of computer science~\cite{Smith76:Wealth}; indeed, it is known that many natural conditions for fairness cannot be achieved in tandem~\cite{Kleinburg16:Inherent}.  Still, group fairness is found in many fielded systems~\cite{Wexler19:WhatIf,Bellamy19:AI}, and we focus on it here.

The study of fairness in MAB was initiated by \citet{Joseph16:FairnessBandits}, who showed for both MAB and CMAB one can implement a fairness definition where within a given pool of applicants, e.g., college admission or mortgages, a worse applicant is not favored over a better one, despite a learning algorithm's uncertainty over the true payoffs.  However, \citet{Joseph16:FairnessBandits} only focus on individual fairness, and do not formally treat the idea of group fairness.  Individual fairness is, in some sense, group fairness taken to an extreme, where every arm is its own singleton group; it offers strong guarantees, but under strong assumptions~\cite{Kearns18:Preventing,Binns20:Apparent}.  

\citet{Celis19:Controlling} propose a bandit-based approach to personalization where arm pulls are constrained to fit some probability distribution defined by a fairness metric such as demographic parity. For example, when recommending news articles, their algorithm provides personalized articles from both left and right sources. Their formulation is perhaps closest in the literature to our formulation as it deals with group fairness, however it does not explicitly assume biased feedback. Instead, it enforces a fair probability distribution without learning about the bias present in the data.


%

There are a number of other recent studies of fairness in the MAB literature.
\citet{chen2020fair} investigate a task allocation setting with a fairness constraint that captures a minimum rate at which a task is assigned to a particular arm; their model is quite general and captures the adversarial and some non-stationary settings.
\citet{liu2017calibrated} look at fairness between arms under the assumption that arm reward distributions are similar (another interpretation of equal treatment of equals).
\citet{patil2019achieving} define fairness such that each arm must be pulled for a predetermined required fraction over the total available rounds.
\citet{claure2019reinforcement} use the MAB framework to distribute resources amongst teammates in human-robot interaction settings; again, fairness is defined as a pre-configured minimum rate that each arm must be pulled.
\citet{Hossain21:Fair} take a more theory-oriented approach to a similar setting, proposing a multi-agent varient of a stochastic MAB setting with a Nash social welfare definition of fairness.
%

Since preliminary versions of this work were presented \cite{Schumann20:Fairness,SLMD:19:FairBandit} there have been several papers that have investigated similar problems. \citet{Wang21:Fairness} look at fairness of exposure in CMAB base systems, specifically focusing on similarity of merit, which is more in line with individual rather than the group fairness we consider here. \citet{Tang21:Generalized} consider a setting inspired by liver transplantation where the objective is to trade off a more egalitarian, max-min, policy in allocating opportunities for surgeons to gain experience in liver transplant training. Finally, \citet{Ron21:Corporate} investigate a setting of allocating opportunities to sub-populations in a corporate decision making setting where each arm needs to pulled at least a budgeted number of times, but where the cost of allocating an opportunity to a non-optimal arm is known in advance. Interestingly, their algorithmt also achieves a $T^{2/3}$ regret, similar to our results.

One needs to be careful when appealing to purely statistical metrics for ensuring fairness.  As argued by \citet{Corbett18MeasureFairness}, simply setting our sights on a form of classification parity, i.e., forcing that some statistical measure be normalized across groups, we may miss bigger picture issues.  Specifically, by only focusing on the statistics of the data we have, we miss an opportunity to identify and understand why the data we have may be \emph{causing} the bias.  Later, we will argue that our novel formalization of regret allows us to actually learn particular sources of bias that may exist in our data.

\section{Preliminaries}

We follow the standard CMAB setting and assume that we are attempting to maximize a measure over a series of time steps $t \in T$. We assume that there is a $d$-dimensional domain for the context space, $\mathcal{X}=\mathbb{R}^d$.  The agent is presented with a set $A$ of arms from which to select and we have $|A| = n$ total arms.  Each of these arms is associated with a, possibly disjoint, context space $\mathcal{X}_i\subseteq \mathcal{X}$.  Additionally, we assume that we have $m$ sensitive groups and that the arms are partitioned into these sensitive groups such that $P_1\cap\cdots\cap P_m=\emptyset$,  $P_1\cup\cdots\cup P_m=A$, and $\forall_{i \in m} |P_i|>1$. For exposition's sake, we assume a binary sensitive attribute with $m=2$ for the remaining of the paper. However, we show the generality of our results to any number of groups in Section~\ref{sec:algorithm}.

Each arm $i$ has a true linear reward function $f_i:\mathcal{X}\rightarrow\mathbb{R}$ such that $f_i(x)=\beta_i\cdot x$ where $\beta_i$ is a vector of coefficients that is unknown to the agent. During each round $t \in T$, a context $x_{t,i}\in\mathcal{X}_i$ is given for each arm $i$. One arm is pulled per round. When arm $i$ is pulled during round $t$, a reward is returned: $r_{t,i}=f_i(x_{t,i})+e_{t,i}$ where $e_{t,i}\sim\mathcal{N}(0,1)$. The goal of the agent is to minimize the regret over all timesteps in $T$.  Formally, the regret of the agent at timestep $t$ is the difference between the arm selected and the best arm that could have been selected.  Let $i^*$ denote the optimal arm that could be selected and $a$ be the selected arm.  Then, the regret at $T$ is

{
\begin{equation}\label{eq:orig_regret}
R(T) = \sum_{t=1}^{T} f(x_{i^*,t}) - f(x_{a,t}).
\end{equation}
}

In this paper we compare our proposed algorithm against three other algorithms: \textsc{TopInterval}, a variation of LinUCB from \citet{Li:10:PersonalRec} with additional annotations to track group membership and treatment of arms, \textsc{NaiveFair} which randomly picks a sensitive group and then applies \textsc{TopInterval} to that group,\footnote{See Section~\ref{sec:naive} for more information} and \textsc{IntervalChaining}, an individually fair algorithm from \citet{Joseph16:Rawlsian}. All algorithms use ordinary least squares (OLS) estimators of the arm coefficients $\hat{\beta}_i$ with a confidence variable $w_{i,t}$ such that the true utility lies within $[\hat{\beta}_i\cdot x_{i,t} - w_{i,t}, \hat{\beta}_i\cdot x_{i,t} + w_{i,t}]$ with probability $1-\delta$. \textsc{NaiveFair} implements a naive version of demographic parity without explicitly looking at societal bias. \textsc{TopInterval} either explores by pulling an arm uniformly at random or exploits by pulling the arm with the highest upper confidence $\hat{\beta}_i\cdot x_{i,t} + w_{i,t}$. To ensure individual fairness, \textsc{IntervalChaining} either explores by choosing an arm uniformly at random or exploits by pulling arms that have overlapping confidence intervals with the arm with the highest upper confidence.


%

\subsection{Regret with Societal Bias}
As mentioned before, ground truth rewards  for sensitive groups can be noisy due to societal or measurement bias. We now formalize this bias in terms of multi-armed bandits. For ease of exposition we assume two groups, but we generalize this in our results.
Again, we assume that $n$ arms can be partitioned into two sets $P_1$ and $P_2$ such that $P_1\cap P_2=\emptyset$ and $P_1\cup P_2=[n]$. We consider $P_1$, with $|P_1| > 1$ as the sensitive set, or set with some societal bias. In this situation, each arm $i$ has another true utility function $f^*(x_{i,t})=\beta_{i}\cdot x_{i,t}$ where $\beta_{i}$ is a vector of coefficients; if arm $i$ is pulled at timestep $t$ the following reward is returned:

{
\begin{equation}\label{eq:new_reward}
r_{i,t}=\beta_{i}\cdot x_{i,t} + \mathbbm{1}[i\in P_1]\psi_{P_1}\cdot x_{i,t} + \mathcal{N}(0,1),
\end{equation}
}

where $\mathbbm{1}[i\in P_1]=1$ when $i\in P_1$ and 0 otherwise, and $\psi_{P_1}$ is a societal or systematic bias against group $P_1$. Note that $\psi_{P_2}$ is a zero vector for the non-sensitive group. Hence, the underlying \emph{biased} utility function can be written as $f(x_{i,t})=\beta_{i}\cdot x_{i,t} + \mathbbm{1}[i\in P_1]\psi_{P_1}\cdot x_{i,t}$.

Using our running example, let's assume that the down payment reward received has some bias against the male applicants compared to the female applicants, while the final repayment does not.  Note that the final repayment is not measured after accepting a loan and is only measured much later. The loan agency should then take the bias into account while learning what `good' applications look like. Or, in a hiring setting, an applicant may have a biased interview (initial reward) while their true performance is measured only after working for a year (later true reward).

We define true regret for pulling an arm $a$ at time $T$ as

{
\begin{equation}\label{eq:new_regret}
    R^*(T) = \sum_{t=1}^{T} f^*(x_{i^*,t}) - f^*(x_{a,t})
\end{equation}
}

where $i^*$ is the optimal arm to pull at timestep $t$ and $f^*(x_{i,t})$ is the true reward with no bias terms $\psi_{P_1}\cdot x_{i,t}$. We also assume that the average true reward (with no bias) for group $P_1$ should be the same as the average reward for group $P_2$. Compare this to Equation~\ref{eq:orig_regret}, which would return the regret on the biased reward function $f(x_{i,t})$.
In the loan agency example, this real regret $f^*(x_{i,t})$ would measure the regret of the final repayments instead of the biased down payment regret.

One can view the societal bias term $\psi_i$ that we learn for some group $i$ as our algorithm learning how to automatically identify and adjust for anti-discrimination for group $i$ compared to all other groups.  Anti-discrimination is the practice of identifying a relevant feature in data and adjusting it to provide fairness under that measure \cite{Corbett18MeasureFairness}.  One example of this, discussed by \citet{Dwork12:Fairness}, \citet{Joseph16:FairnessBandits}, and in the official White House algorithmic decision making statement \cite{Munoz:16:BigData}, comes up in college admissions.  Given other factors, specifically income level, some colleges weight SAT scores \emph{less} in wealthy populations due to the presence of tutors while increasing the weight of working-class populations~\cite{Belkin19:SAT}.  While in these admissions settings the adjustments may be ad-hoc, we learn our bias term from data. Past work has compared the vector $\beta$ learned for each arm as akin to adjusting for these biases \cite{Dwork12:Fairness}.  While this is true at an \emph{individual} level, our explicit modeling of bias allows us to discover these adjustments at a \emph{group} level. 




\section{Group Fair Contextual Bandits}
\label{sec:algorithm}
In this section, given our new definition of reward (Equation~\ref{eq:new_reward}) and corresponding new definition of regret (Equation~\ref{eq:new_regret}), we present the algorithm \textsc{GroupFairTopInterval} (Algorithm~\ref{alg:fair_top_interval}) which takes societal bias into account.  We also give a bound on its regret in this new reward and regret setting.
Subsequently, we briefly describe the algorithm.

In \textsc{GroupFairTopInterval}, each round $t$ is randomly chosen with probability $\frac{1}{t^{1/3}}$ to be an exploration round. The exploration round randomly chooses an arm to pull.



The remaining rounds become exploitation rounds, where linear estimates are used to pull arms. \textsc{GroupFairTopInterval} learns two different types of standard OLS linear estimators~\cite{Kuan04:Classical}. The first is a coefficient vector $\hat{\beta}_{i,t}$ for each arm $i$ (line~\ref{line:beta}). Additionally, \textsc{GroupFairTopInterval} learns a group coefficient vector $\hat{\psi}_{P_j,t}$ for each group $P_j$ (lines~\ref{line:psi_1} and \ref{line:psi_2}). To calculate these coefficient vectors, the algorithm keeps track of previous arm pull rewards for each arm $i$ at every timestep $t$ in a vector $Y_{i,t}$, and the corresponding contexts for each arm pull in a matrix $X_{i,t}$. A similar vector $\mathcal{Y}_{P_j,t}$ and matrix $\mathcal{X}_{P_j,t}$ is kept for both groups $P_j$. As mentioned previously, we treat $P_1$ as the sensitive group of arms. An arm $i$ in the non-sensitive group $P_2$ has a reward estimation of $\hat{\beta}_{i,t}\cdot x_{i,t}$, while an arm $i$ in the sensitive group $P_1$ has a bias corrected reward estimation of 
$\hat{\beta}_{i,t}\cdot x_{i,t} - \hat{\psi}_{P_1,t}+\hat{\psi}_{P_2,t}$.

For each arm $i$, the algorithm calculates confidence intervals $w_{i,t}$ around the linear estimates $\hat{\beta}_{i,t}\cdot x_{i,t}$ using a Quantile function $Q$ (line~\ref{line:w}). This means that the true utility (including some bias) falls within $[\hat{\beta}_{i,t}\cdot x_{i,t} - w_i, \hat{\beta}_{i,t}\cdot x_{i,t}+w_i]$ with probability $1-\delta$ at every arm $i$ and every timestep $t$. Similarly, for each group $P_j$ and context $w_{i,t}$ for a given arm $i$ at timestep $t$, the algorithm calculates a confidence interval $b_{P_j,i,t}$ using a Quantile function $Q$ (lines~\ref{line:psi_1} and \ref{line:psi_2}). This means that the true \emph{group} utility (or true average group utility) falls within $[\hat{\psi}_{P_j,i,t}\cdot x_{i,t} - b_{P_j,i,t}, \hat{\psi}_{P_j,i,t}\cdot x_{i,t} + b_{P_j,i,t}]$ with probability $[1-\delta]$. Using the confidence intervals $w_{i,t}$ and $b_{P_j,i,t}$, and the linear estimates $\hat{\beta}_{i,t}\cdot x_{i,t}$ and $\hat{\psi}_{P_j,i,t} \cdot x_{i,t}$, we calculate the upper bound of the estimated reward for each arm $i$ (lines~\ref{line:upper_1} and \ref{line:upper_2}), pulling the arm with the highest upper bound (line~\ref{line:pull}).

\begin{algorithm}[h]  
\caption{\textsc{GroupFairTopInterval}}\label{alg:fair_top_interval}
\begin{algorithmic}[1]
\REQUIRE $\delta$, $P_1$, $P_2$
\FOR{$t=1\ldots T$}
\STATE With probability $\frac{1}{t^{1/3}}$, play $i_t\in_R\{1,\ldots , n\}$ and observe reward $y_{i_t,t}$
\STATE \textbf{otherwise:}
\begin{ALC@g}
\STATE $\hat{\psi}_{P_1,t} \defeq \left(\mathcal{X}_{P_1,t}^T\mathcal{X}_{P_1,t}\right)^{-1}\mathcal{X}_{P_1,t}^T\mathcal{Y}_{P_1,t-1}$ \label{line:psi_1}
\STATE $\hat{\psi}_{P_2,t} \defeq \left(\mathcal{X}_{P_2,t}^T\mathcal{X}_{P_2,t}\right)^{-1}\mathcal{X}_{P_2,t}^T\mathcal{Y}_{P_2,t-1}$ \label{line:psi_2}
\FOR{$i=1\ldots n$}
\STATE $\hat{\beta}_{i,t} \defeq \left(X_{i,t}^TX_{i,t}\right)^{-1}X_{i,t}^TY_{i,t-1}^T$ \label{line:beta}
\STATE $F_{i,t} \defeq \mathcal{N}\left(0,\sigma^2x_{i,t}\left(X_{i,t}^TX_{i,t}\right)^{-1}x_{i,t}^T\right)$
\STATE $w_{i,t} \defeq Q_{F_{i,t}}\left(\frac{\delta}{2nt}\right)$ \label{line:w}
\IF{$i\in P_1$}
\STATE $\mathcal{F}_{P_1,i,t} \defeq \mathcal{N}\left(0,\sigma^2x_{i,t}\left(\mathcal{X}_{P_1,t}^T\mathcal{X}_{P_1,t}\right)x_{i,t}^T\right)$
\STATE $\mathcal{F}_{P_2,i,t} \defeq \mathcal{N}\left(0,\sigma^2x_{i,t}\left(\mathcal{X}_{P_2,t}^T\mathcal{X}_{P_2,t}\right)x_{i,t}^T\right)$
\STATE $b_{P_1,i,t} \defeq Q_{\mathcal{F}_{P_1,i,t}}\left(\frac{\delta}{2\frac{n}{|P_1|}T}\right)$\label{line:b_1}
\STATE $b_{P_2,i,t} \defeq Q_{\mathcal{F}_{P_2,i,t}}\left(\frac{\delta}{2\frac{n}{|P_2|}T}\right)$\label{line:b_2}
\STATE $\hat{u}_{i,t} \defeq \hat{\beta}_{i,t}\cdot x_{i,t} + w_{i,t} - \hat{\psi}_{P_1,t} \cdot x_{i,t} + b_{P_1,i,t} + \hat{\psi}_{P_2,t}\cdot x_{i,t} + b_{P_2,i,t}$\label{line:upper_1}
\ELSE
\STATE $\hat{u}_{i,t} \defeq \hat{\beta}_{i,t}\cdot x_{i,t} + w_{i,t}$\label{line:upper_2}
\ENDIF
\ENDFOR
\STATE Play $\argmax_i\hat{u}_{i,t}$ and observe reward $y_{i,t}$\label{line:pull}
\end{ALC@g}
\ENDFOR
\end{algorithmic}
\end{algorithm}

Returning to our running example, using \textsc{GroupFairTopInterval}, the loan agency would learn a down payment reward function for each of the arms, i.e., a coefficient vector $\beta_i$ where $i\in[$young female arm, young male arm, older female arm, older male arm$]$, as well as the group average coefficients for the gender-grouped arms, $\psi_{P_j}$, for male and female. Using the gender-grouped coefficients, expected rewards for male arms are reweighted to account for the bias in down payment.

Standard algorithms like \textsc{TopInterval}\footnote{A variant of the contextual bandit LinUCB by \citet{Li:10:PersonalRec}} would choose an arm $i=\argmax(\hat{\beta}\cdot x_{i,t} + w_{i,t})$, ignoring societal bias (Equation~\ref{eq:new_reward}, leading to a larger true regret (Equation~\ref{eq:new_regret})). Note that \textsc{GroupFairTopInterval} can be extended to multiple groups by defining an overall average reward.

\textsc{GroupFairTopInterval} is fair---in the context of our group fairness definitions---and satisfies the following theorem.
Appendix~B of the full paper
provides a detailed, complete proof.


\begin{theorem}\label{thm:two_group_bound}
For two groups $P_1$ and $P_2$, where $P_1$ has a bias offset in rewards, \textsc{GroupFairTopInterval} has regret
\begin{equation*}
\resizebox{.9\hsize}{!}{%
    $R^*(T) = O\left( \sqrt{\frac{dn\ln\frac{2nT}{\delta}}{l}}T^{2/3}  + \left( \frac{dnL}{l}\left(\ln^2\frac{2nT}{\delta}+\ln d\right)\right)^{2/3} \right)$.%
    }
\end{equation*}
\end{theorem}

\begin{hproof}
We start by proving two lemmas. The first of which states that with probability at least $1-\delta$:
{\small
\begin{equation}
    \left| \hat{\beta}_{i,t} \cdot x_{i,t} - (\beta_i \cdot x_{i,t} +  \mathbbm{1}[i\in P_1]\psi_{P_1}\cdot x_{i,t}) \right|\leq w_{i,t}
\end{equation}
}
holds for any $i$ at time $t$. Similarly, the second states that with probability at last $1-\delta$:
{\small
\begin{equation}
    \left| \hat{\beta}_{i,t} \cdot x_{i,t} - \beta_i \cdot x_{i,t}\right|\leq w_{i,t}
\end{equation}
}
holds for any group $P_j$, any arm $i$, and at any timestep $t$.
By combining these two lemmas, we can see that arms should be treated fairly.

The regret for \textsc{GroupFairTopInterval} can be broken down into three terms:
\begin{align}\label{eq:ps_3split}
    R^*(T) &= \sum_{t:\ t\textit{ is an explore round}} \mathit{regret}(t)  \\
    &+ \sum_{t:\ t\textit{ is an exploit round and } t<T_1}\mathit{regret}(t) \nonumber \\
    &  + \sum_{t:\ t\textit{ is an exploit round and }t\geq T_1}\mathit{regret}(t).
\end{align} 

First, for any $t$ we have:
{\small
\begin{equation}\label{eq:ps_term1}
    \sum_{t'<t}\frac{1}{t^{1/3}}=\Theta(t^{2/3}).
\end{equation}
}
We then show that the number of rounds $T_1$ after which we have sufficient samples such that the estimators are well concentrated is:
{\small
\begin{equation}\label{eq:ps_term2}
    T_1=\Theta\left(\min_{a}\left(\frac{dnL}{\lambda_{min_{a,d}}}\left(\ln^2\frac{2}{\delta} + \ln d\right)\right)^{3/2}\right).
\end{equation}
}
Finally, we bound the third term in Equation~\ref{eq:ps_3split} as follows:
{\small
\begin{align}\label{eq:ps_term3}
    &\sum_{t:\ t\textit{ is an exploit round and }t\geq T_1}regret(t) \\
    &\leq O\left(\sqrt{dn\frac{\ln\frac{2nT}{\delta}}{\min_i\lambda_{\min_{i,d}}}}T^{2/3} + \delta'T\right).
\end{align}
}
Combining Equations~\ref{eq:ps_3split}, \ref{eq:ps_term1}, \ref{eq:ps_term2}, and \ref{eq:ps_term3}, we have Theorem~\ref{thm:two_group_bound}.
\end{hproof}

Note that we can extend Algorithm~\ref{alg:fair_top_interval} to $m$ groups. In this setting, we make the strong assumption that true rewards are centered about $\rho$ defined by the user.\footnote{See 
Appendix~B.2
for further details.} In this adaption of the algorithm, we set the upper bound radius for arm $i$ as:

{
\begin{equation*}
    \hat{u}_{i,t}=\hat{\beta}_{i,t}\cdot x_{i,t} + w_{i,t} + \rho - \hat{\psi}_{P_j,t} \cdot x_{i,t} + b_{P_j,i,t}
\end{equation*}
}

where $i\in P_j$. We then have the following theorem for multiple groups:
\begin{theorem}
For $m$ groups $P_1,\ldots,P_m$, 
\textsc{GroupFairTopInterval (Multiple Groups)} has regret
\begin{equation*}\label{eq:mult_group_thm}
\resizebox{.9\hsize}{!}{%
   $R^*(T)\ =\ O\left( \sqrt{\frac{dn\ln\frac{2nT}{\delta}}{l}}T^{2/3} + \left( \frac{dnmL}{l}\left(\ln^2\frac{2nT}{\delta}+\ln d\right)\right)^{2/3} \right)$.%
    }
\end{equation*}
where $l=\min_i\lambda_{min_{i,d}}$, with $\lambda_{min_{i,d}}$ the smallest eigenvalue of $X_{i,t}^T X_{i,t}$; and $L>\max_t\lambda_{\max}(x_{i,t}^Tx_{i,t})$.
\end{theorem}

\section{Experiments}
\begin{figure*}
\centering
\begin{subfigure}[t]{0.18\textwidth}
   \includegraphics[width=1\linewidth]{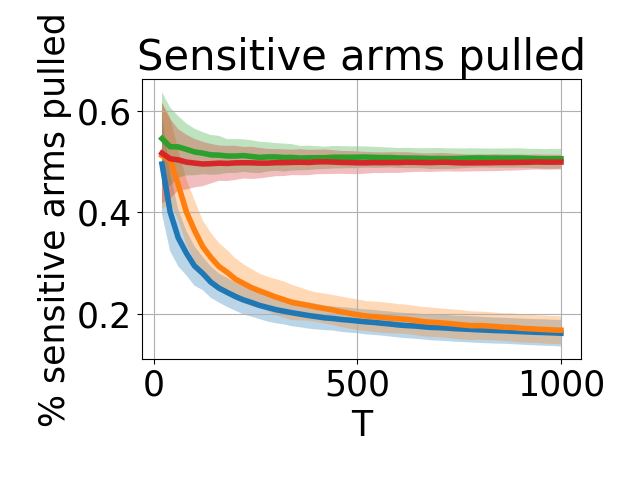}
   \caption{Increasing the total budget $T$, for $n=10$, $\mu=10$, and number of sensitive arms = $5$}
   \label{fig:pulls_T} 
\end{subfigure}
\quad
\begin{subfigure}[t]{0.18\textwidth}
   \includegraphics[width=1\linewidth]{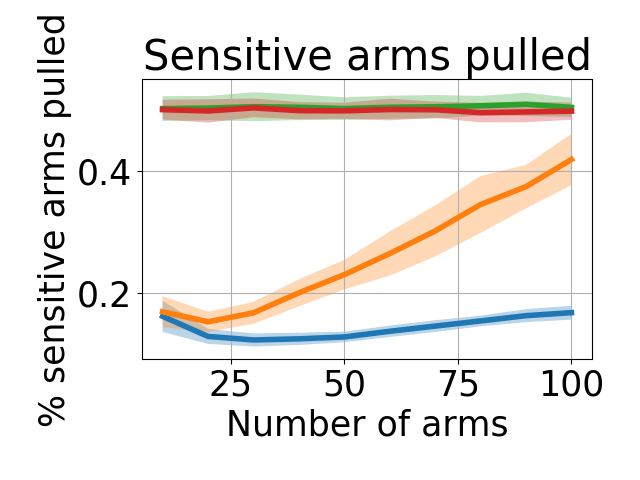}
   \caption{Increasing the number of arms $n$, for $T=1000$, $\mu=10$, and number of sensitive arms = $5$}
   \label{fig:pulls_arms}
\end{subfigure}
\quad
\begin{subfigure}[t]{0.18\textwidth}
   \includegraphics[width=1\linewidth]{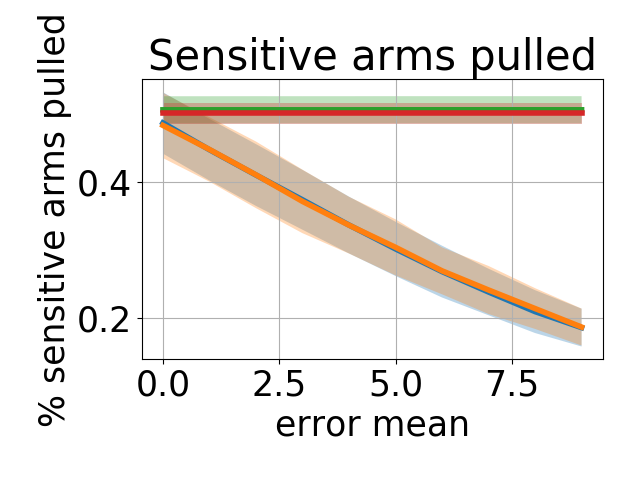}
   \caption{Increasing $\mu$, for $n=10$, $T=1000$, and number of sensitive arms = $5$}
   \label{fig:pulls_error}
\end{subfigure}
\quad
\begin{subfigure}[t]{0.18\textwidth}
   \includegraphics[width=1\linewidth]{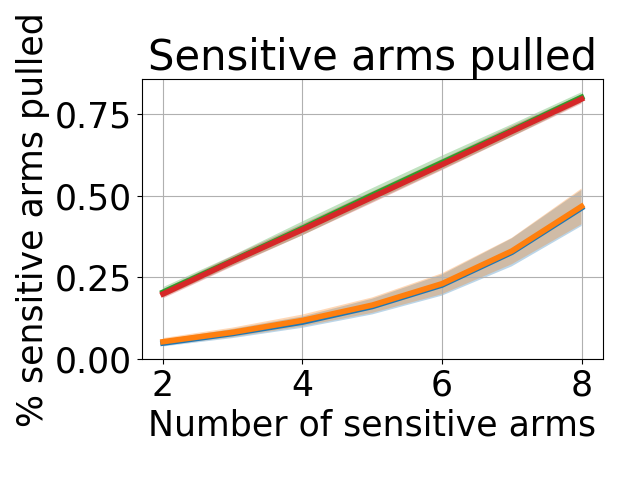}
   \caption{Increasing the fraction of overall sensitive arms, for $n=10$, $T=1000$, $\mu=10$}
   \label{fig:pulls_ratio}
\end{subfigure}
\quad
\begin{subfigure}[t]{0.18\textwidth}
   \includegraphics[width=1\linewidth]{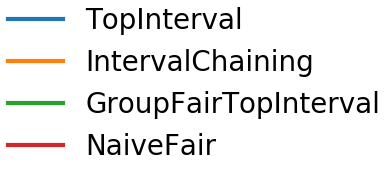}
   \caption{Legend}
   \label{fig:pulls_legend}
\end{subfigure}

\caption{Percentage of total arm pulls that were pulled using sensitive arms.}
\label{fig:synth_pulls}
\end{figure*}

\begin{figure*}
\centering
\begin{subfigure}[t]{0.18\textwidth}
   \includegraphics[width=1\linewidth]{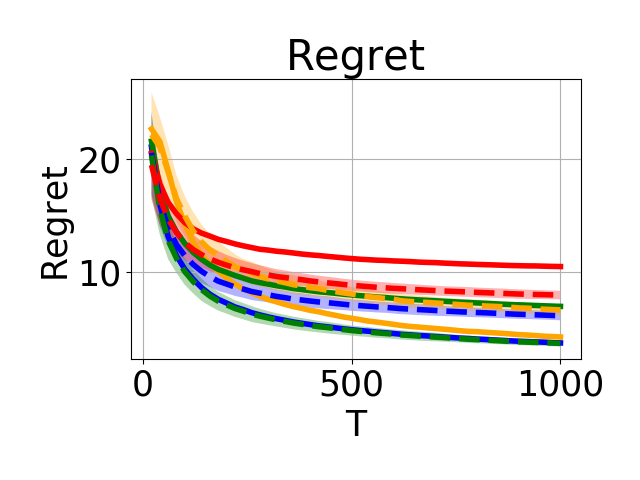}
   \caption{$n=10$, $\mu=10$, \# of sensitive arms~=~5}
   \label{fig:regret_T} 
\end{subfigure}
\quad
\begin{subfigure}[t]{0.18\textwidth}
   \includegraphics[width=1\linewidth]{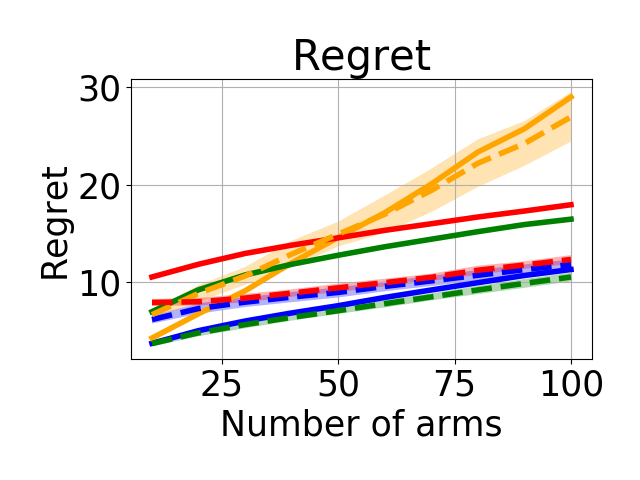}
   \caption{$T=1000$, $\mu=10$, \# of sensitive arms~=~5}
   \label{fig:regret_arms}
\end{subfigure}
\quad
\begin{subfigure}[t]{0.18\textwidth}
   \includegraphics[width=1\linewidth]{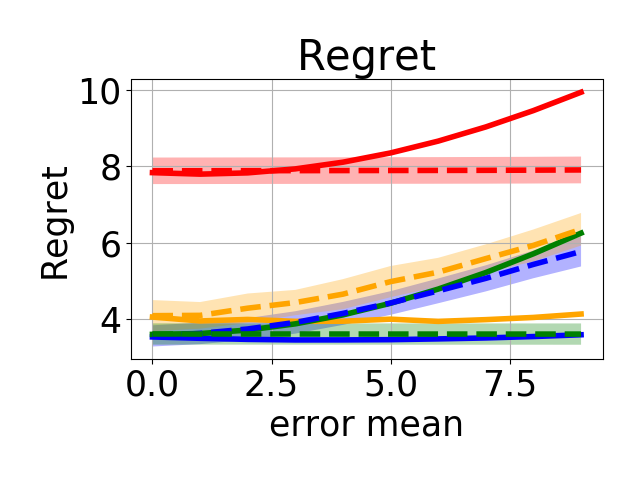}
   \caption{$n=10$, $T=1000$, \# of sensitive arms~=~5}
   \label{fig:regret_error}
\end{subfigure}
\quad
\begin{subfigure}[t]{0.18\textwidth}
   \includegraphics[width=1\linewidth]{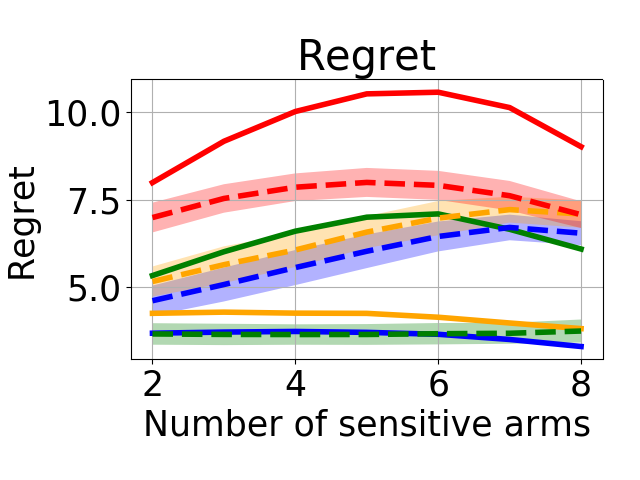}
   \caption{$n=10$, $T=1000$, $\mu=10$}
   \label{fig:regret_ratio}
\end{subfigure}
\quad
\begin{subfigure}[t]{0.18\textwidth}
   \includegraphics[width=1\linewidth]{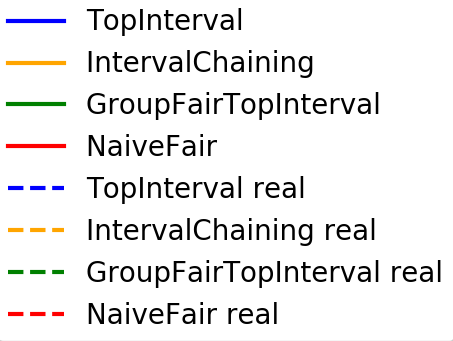}
   \caption{Legend}
   \label{fig:regret_legend}
\end{subfigure}

\caption{Regret for synthetic experiments. The solid lines are regret given the rewards received from pulling the arms (including the group bias). The dashed lines is the true regret (without the group bias).}
\label{fig:synth_regret}
\end{figure*}
To empirically evaluate \textsc{GroupFairTopInterval}, we perform experiments on synthetic data to demonstrate the effects of various parameters, and on real datasets to demonstrate how \textsc{GroupFairTopInterval} performs in the wild. In each of these sections we compare to \textsc{TopInterval}, due to~\citet{Li:10:PersonalRec}, \textsc{NaiveFair} (See Section~\ref{sec:naive}), and \textsc{IntervalChaining}, due to~\citet{Joseph16:FairnessBandits}.

\subsection{\textsc{NaiveFair}}\label{sec:naive}
One popular definition of group fairness in classification is the notion of demographic parity.  Formally, given a protected demographic group $A$, we want:

{
 \begin{equation}
     \Pr(\hat{Y}=1 | A=0) = \Pr(\hat{Y}=1 | A=1),
 \end{equation}
}
 
where the probability of assigning a classification label $\hat{Y}=1$ does not change based on the sensitive attribute class $A$. Demographic parity is important when ground truth classes $Y$ are extremely noisy for sensitive groups due to some societal or measurement bias. Assume that we have a classifier that predicts whether an individual should receive a loan where our sensitive attribute $A$ is binary gender. Demographic parity states that the probability of getting a loan should be the same for males ($A= 0$) and females ($A=1$).

In converting this definition of demographic parity to the the multi-armed bandit setting, we alter the definition to be that the probability of pulling an arm $a$ does not change based on group membership $P_j$: 

{
\begin{equation}
    \Pr(\textit{pull } a | a\in P_0) = \Pr(\textit{pull } a | a\in P_1).
\end{equation}
}


Continuing our running example, assume we are a loan agency. The loan agency receives 4 applications at every timestep $t$: an applicant from a young female, an applicant from a young male, an applicant from a older female, an applicant from an older male; we must choose one application to grant at each timestep. After granting a loan the loan agency receives a down payment on that loan as reward. This reward is then used to update the estimates of whether or not a ``good'' loan application was received for the pulled arm. Assume that the loan agency wants to act fairly using the binary sensitive attribute of gender.  Then, the probability that the loan agency chooses a female applicant at timestep $t$ should be the same as the probability of choosing a male applicant.

\begin{algorithm}
\caption{\NaiveGroupFair{}}\label{alg:naive}
\begin{algorithmic}[1]
\REQUIRE $\delta$, $P_1$, $P_2$
\FOR{$t=1\ldots T$}
\STATE $P \gets$ Randomly choose group $P_1$ or $P_2$.
\STATE Pull arm in $P$ based on \textsc{TopInterval}
\ENDFOR
\end{algorithmic}
\end{algorithm}

A naive algorithm to enforce this definition of fairness is defined in Algorithm~\ref{alg:naive}. We first pick from the groups uniformly at random, and then apply a regular CMAB algorithm like \textsc{TopInterval}\footnote{\textsc{TopInterval} is a variant of LinUCB by \citet{Auer:02:UCB}.} or \textsc{ContextualThompsonSampling} \cite{Agrawal:13:ThompsonSampling} to choose which arm to pull within the group. Using our running example, \NaiveGroupFair{} would randomly pick between male or female, and then choose the best applicant between the younger and older pair.

\subsection{Synthetic Experiments}\label{sec:synth}

In each synthetic experiment, we generate true coefficient vectors $\beta_i$ by choosing coefficients uniformly at random for each arm $i$. Contexts at each timestep $t$ are chosen randomly for each arm $i$. Bias coefficients $\psi_1$ are set uniformly at random with mean $\mu=10$. Seeds are set at the beginning of each experiment to keep arms consistent between algorithms.

\begin{figure*}
\centering
\begin{subfigure}[t]{0.23\textwidth}
   \includegraphics[width=1\linewidth]{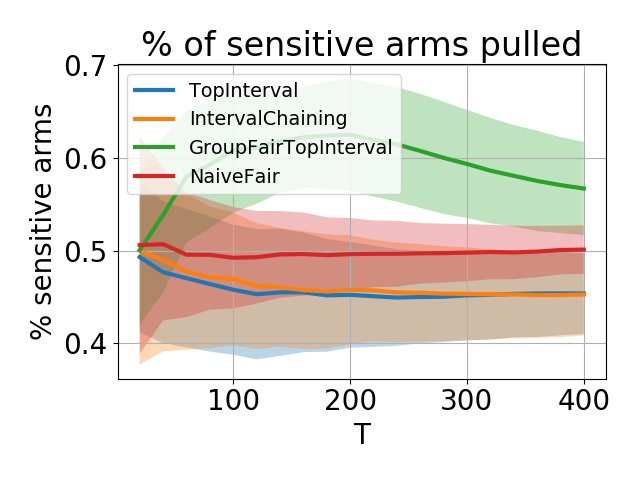}
   \caption{Sensitive arm pulls (\%)}
   \label{fig:family_pulls}
\end{subfigure}
\quad
\begin{subfigure}[t]{0.23\textwidth}
   \includegraphics[width=1\linewidth]{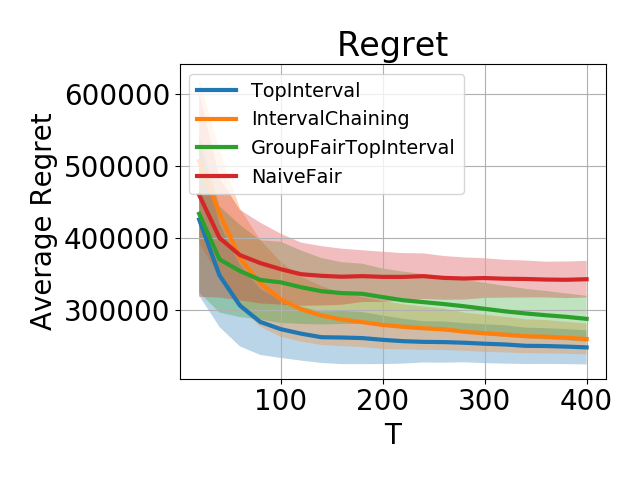}
   \caption{Regret}
   \label{fig:family_regret} 
\end{subfigure}
\quad
\begin{subfigure}[t]{0.23\textwidth}
   \includegraphics[width=1\linewidth]{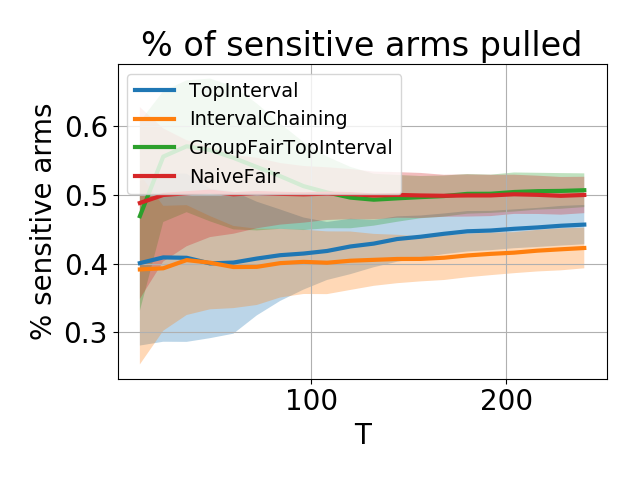}
   \caption{Sensitive arm pulls (\%)}
   \label{fig:compas_pulls}
\end{subfigure}
\quad
\begin{subfigure}[t]{0.23\textwidth}
   \includegraphics[width=1\linewidth]{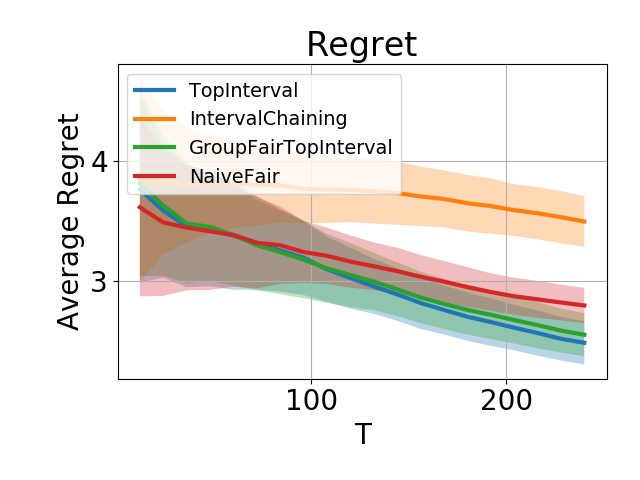}
   \caption{Regret}
   \label{fig:compas_regret} 
\end{subfigure}

\caption{Results of running contextual bandit algorithms on the family income and expenditure dataset (Figures~\ref{fig:family_pulls} and \ref{fig:family_regret}), as well as the COMPAS dataset (Figures~\ref{fig:compas_pulls} and \ref{fig:compas_regret}). Figures~\ref{fig:family_pulls} and \ref{fig:compas_pulls} show the percentage of pulls that were of sensitive arms. Figures~\ref{fig:family_regret} and \ref{fig:compas_regret} show the biased regret for each of the algorithms. Note that the ``real'' regret like that shown in the synthetic experiments cannot be calculated.}
\end{figure*}

We run four different types of experiments:\footnote{Additional experiments can be found in 
Appendix~C.%
}
\begin{inparaenum}[\bfseries a\normalfont)]
    \item Varying the total budget for pulling arms ($T$) while setting number of arms $n=10$, error mean $\mu=10$, number of sensitive arms equal to 5, and context dimension $d=2$ (Figures~\ref{fig:regret_T} and \ref{fig:pulls_T}).
    \item Varying the total number of arms $n$ while setting total budget $T=1000$, error mean $\mu=10$, ratio of sensitive arms to 50\%, and context dimension $d=2$ (Figures~\ref{fig:regret_arms} and \ref{fig:pulls_arms}).
    \item Varying the error mean $\mu$ while setting total budget $T=1000$, number of arms $n=10$, number of sensitive arms equal to 5, and context dimension $d=5$ (Figures~\ref{fig:regret_error} and \ref{fig:pulls_error}).
    \item Varying the number of sensitive arms while setting total budget $T=1000$, number of arms $n=10$, error mean $\mu=10$, and context dimension $d=2$ (Figures~\ref{fig:regret_ratio} and \ref{fig:pulls_ratio}).
\end{inparaenum}

The plots in Figure~\ref{fig:synth_pulls} show the percentage of times an algorithm pulled a sensitive arm over the full budget $T$. In order to be fair, the percentage of sensitive arms pulled should be proportional to the number of sensitive arms, i.e., when there are 5 sensitive arms out of 10 total, the percentage of sensitive arms pulled is roughly 50\%. Figure~\ref{fig:synth_regret} shows the perceived regret that includes bias $\psi$ as solid lines, and real regret that corrects bias (see Equations~\ref{eq:new_reward} and \ref{eq:new_regret}) as dashed lines. Algorithms with low real regret are considered `good'.

Figure~\ref{fig:pulls_T} shows that once exploration is over, \textsc{GroupFairTopInterval} pulls sensitive arms roughly 50\% of the time, matching the 50\% of sensitive arms. Figure~\ref{fig:regret_T} shows that \textsc{GroupFairTopInterval} performs comparably on real regret as \textsc{TopInterval} performs on biased regret. This means \textsc{GroupFairTopInterval} should be used over \textsc{TopInterval} in contexts where bias is anticipated.  \textsc{NaiveFair} performs poorly in the context of societal bias.

Figure~\ref{fig:pulls_arms} illustrates that \textsc{IntervalChaining} becomes more group fair as the number of arms increase. This is because many arms are chained together and therefore, arms are chosen uniformly at random. Figure~\ref{fig:regret_arms} illustrates this random picking of arms as real regret and biased regret increases dramatically for \textsc{IntervalChaining}. 

As expected, Figure~\ref{fig:pulls_error} illustrates that when the error mean  $\mu$ is large, both \textsc{IntervalChaining} and \textsc{TopInterval} choose fewer sensitive arms. This leads to a high real regret as shown in Figure~\ref{fig:regret_error}. Following \citet{Kleinburg16:Inherent}, Figure~\ref{fig:regret_error} also suggests that one cannot have both individual \emph{and} group fairness in a scenario with high mean error. The randomness in \textsc{NaiveFair} leads to a very high regret for both perceived regret and real regret.

Figure~\ref{fig:pulls_ratio} demonstrates the fairness property of proportionality. The percentage of sensitive arms pulled by \textsc{GroupFairTopInterval} matches the number of sensitive arms. As shown in Figure~\ref{fig:regret_ratio}, the number of sensitive arms does not affect the real regret of \textsc{GroupFairTopInterval}.



\subsection{Experiments on Real-World Data}

After exploring \textsc{GroupFairTopInterval} on synthetic data, we move on to using both the Philippines family income and expenditure dataset on Kaggle\footnote{\tiny https://www.kaggle.com/grosvenpaul/family-income-and-expenditure} and the ProPublica COMPAS dataset.\footnote{\tiny https://www.kaggle.com/danofer/compass} When one looks at the gender and age breakdown in the family income dataset, one can see that quite often female heads of households make more money than males in the Philippines. This is most likely due to the large number of Filipino women who work out of the country; it is estimated that up to 20\% of the GDP of the Philippines is actually remittances from these overseas---primarily female---workers.\footnote{\tiny https://www.nationalgeographic.com/magazine/2018/12/filipino-workers-return-from-overseas-philippines-celebrates/}
In fact, almost 60\% of overseas workers are women and 75\% of these women are between the ages of 25 and 44.\footnote{\tiny https://psa.gov.ph/content/2017-survey-overseas-filipinos-results-2017-survey-overseas-filipinos} In the COMPAS dataset ProPublica observed a societal bias over recidivism risk scores for African-Americans.\footnote{\tiny https://www.propublica.org/article/machine-bias-risk-assessments-in-criminal-sentencing}

\paragraph{Experimental Setup.}
Given the skew of high income coming from female head of households in the family income dataset, we treat the binary `Household Head Sex' feature as the sensitive attribute. To create arms, we split up households based on `Household Head Age' bucketed into the following five groups: (8, 27], (27, 45], (45, 63], (63, 81], and (81, 99]. We then have 10 different arms (for example, two arms would be Female head of household between 8 and 27, and Male head of household between 8 and 27). 

Similarly, we treat African-American individuals from the COMPAS dataset as the sensitive attribute. We create arms by splitting up households based on the three age categories found in the data. We therefore have six different arms.

At each timestep $t$, we randomly select an individual from each arm. The context vector is the remaining features where any nominal features are transformed into integers. After an arm is pulled, a reward of the household income (for the family income dataset) or violent decile score (for COMPAS) is returned. We use these datasets for illustrative purposes.


\paragraph{Results.}
We see the same behavior of arm pulls in the real world data. Figures~\ref{fig:family_pulls} and~\ref{fig:compas_pulls} show that after a period of exploration, the percentage of sensitive arms (male-grouped arms) pulled gets very close to 50\%, matching the proportion of sensitive-grouped arms. 

Figures~\ref{fig:family_regret} and~\ref{fig:compas_regret} are perhaps more interesting. Since we cannot measure the ``real'' regret without the bias we assumed from the sensitive-grouped arms, we consider the gap between \textsc{GroupFairTopInterval} and \textsc{TopInterval} as the price of fairness. The gap in regret is small compared to the increase in percentage of sensitive arms pulled. However, the gap in regret for \textsc{NaiveFair} is large in comparison. This suggests that explicitly learning a societal bias term will help in biased settings with low price to perceived regret. Note that there is a difference between regret scales for the two different datasets. This is due to the family income and expenditure dataset reporting regret in income, while the COMPAS dataset reports regret in recidivism score.

\section{General Discussion \& Ethical Implications of the Work}\label{ap:ethics}

This work was directly motivated by research into bias found in machine learning models. There have also been calls to action for more research to be done on bias mitigation in online learning settings~\cite{Chouldechova:18:FrontiersFairness,Bogen18:Help}, specifically in multi-armed bandit settings~\cite{Schumann20:Fairness} and related areas such as recommender systems~\cite{Singh20:Building,Burke:AlgoFairness,Burke:RecFair2020}. Directly addressing these calls, in this work we propose a method of alleviating societal or measurement bias introduced into reward feedback. Using our CMAB model should help mitigate biased behaviors found in bandit systems currently in use~\cite{Sweeney:2013:AdDelivery}. 

Additionally, as noted by \citet{ONeil16:Weapons}, models can provide a biased feedback loop. We hope that by incorporating a societal bias term we can learn something about the bias that is being introduced. The coefficient vector will show which features are incorporating bias into the model. This allows users to address these features outside of the model and potentially find the sources of the societal bias. We do note that addressing societal bias and fixing the solution is a nontrivial task, the societal bias term provides the initial step of measurement.

On the other hand, as noted by~\citet{Schumann20:Fairness},~\citet{Shneiderman20:Human}, and many others, humans should still be active participants in decision making. If models such as our CMAB model are used to replace more and more human decision makers, this could have unintended and potentially negative medium- and long-term side effects. All models should be monitored for biased feedback loops in the given contexts that they are being used~\cite{ONeil16:Weapons}.

Choosing a particular definition of fairness---conditioned on deciding that it is even appropriate to formally define a notion of fairness in the first place---is a morally-laden decision.  We note that, as machine learning practitioners, in many societally-relevant applications it is paramount that we maintain an open dialogue with stakeholders.  In this work, we analyze a sequential decision-making system under one particular definition, group fairness; it is certainly not the case that this is a one-size-fits-all solution that would be deployable without receiving input from that larger set of stakeholders.  Indeed, recent research~\cite{Saha20:Measuring} shows that non-expert users may have vastly varying degrees of comprehension of different definitions of fairness, and that the degree of comprehension may be a function in part of education level and other features that may correlate with measures of marginalization; this hints that the consequences of incorporating fairness definitions into machine-learning-based systems may not be uniformly understood by participants, and indeed that those participants who may be impacted the most by that change could comprehend that potential impact the least.  In an allocative system like the one we describe in the main paper, nuanced considerations must be considered.

Our new definitions of reward (Equation~\ref{eq:new_reward}) and regret (Equation~\ref{eq:new_regret}) for the MAB setting provide an opportunity to look at biased data in a new light. In many cases, ground truths provided during learning are noisy with respect to sensitive groups. Additionally, debiased ground truths may be very expensive to receive or may take a long time to acquire. For instance, if looking at loans, true rewards of repayment may take years to receive. Or, for example, in hiring---the true reward of hiring an individual may take over a year to estimate, while the initial estimate may be influenced by a hiring team's unconscious bias over features such as ethnicity, gender, or orientation.

Our proposed algorithm, \textsc{GroupFairTopInterval}, learns societal bias in the data while still being able to differentiate between individual arms.
Previous solutions relied on setting ad-hoc thresholds, requiring some form of quota, or choosing groups uniformly at random. While it is true that \textsc{GroupFairTopInterval} can easily be extended to a case where we know that the average of a group is a constant offset from the other group. That being said, defining such offsets raises a host of other ethical questions. For instance, in the US, the EEOC (Equal Employment Opportunity Commision) poposed that the ratio of the most favored group compared to the least favored group must not be less than 0.8. Meanwhile, this type of comparison is forbidden in some countries~\cite{Lieberman01:Tale}.  In any case, these prior solutions either lead to high regret, or require a large amount of domain knowledge for the chosen application.

\section{Conclusion \& Future Research}

This paper explores group fairness in the contextual multi-armed bandit (MAB) setting. Our main contributions are:
    (1) we provide a new definition of reward and regret which captures societal bias;
    (2) we provide an algorithm that learns and corrects for that definition of societal bias; and
    (3) we empirically explore the effects different CMAB algorithms have in the setting of societal bias.

Future work could expand \textsc{GroupFairTopInterval} to enforce individual fairness within groups. Intersectional group fairness is also important to look at in the MAB setting where more than one type of sensitive attribute needs to be protected. Additionally, other group fairness definitions such as Equalized Opportunity should be converted to the MAB setting~\cite{Hardt16:Equality}. Another interesting direction for future work is to mix ideas from the study of budget constrained bandits \cite{ding2013multi,wu2015algorithms} with our fairness definitions. We have also assumed individual arms have fixed group membership; generalizing to a setting where memberships in protected groups may change at every timestep $t$ would fit more real world applications.




\begin{acks}
Dickerson and Schumann were supported by 
NSF CAREER Award IIS-1846237, 
NSF Award CCF-1852352, 
NSF Award SMA-2039862,
NIST MSE Award \#20126334, 
DARPA GARD \#HR00112020007, 
DARPA SI3-CMD \#S4761, 
DoD WHS Award \#HQ003420F0035, 
and ARPA-E DIFFERENTIATE Award \#1257037.  Mattei was supported by NSF Awards IIS-RI-2007955 and IIS-III-2107505.  We thank Aviva Prins and Aravind Srinivasan for helpful comments on earlier versions of this paper, and the anonymous reviewers for helpful comments and constructive critiques.
\end{acks}

{
\bibliographystyle{ACM-Reference-Format} 
\balance
\bibliography{ref}
}

\clearpage
\appendix


\section{Additional Related Research}\label{ap:relwork}
We now discuss, and appropriately compare and contrast, additional related research in settings similar to some aspects of our own model.

\noindent\textbf{Fairness in ranking.}  A closely related area to our work is the research into fairness in rankings \cite{Singh:2018:ExposureRankings}, multi-stakeholder recommender systems \cite{Abdollahpouri:2019:Recommendation}, and item allocation \cite{benabbou2018diversity,Benabbou:2019:Housing}.  When algorithms return rankings for an individual to select from, one must pay attention to the ordering and the positioning of various groups \cite{Singh:2018:ExposureRankings}.  One can see this as an application of the group fairness concept to the slates that are chosen for display.  A particular aspect of recommendation systems that one needs to keep in mind is that often there are different stakeholders: the person receiving the recommendation, the company giving the recommendation, and the businesses that are the subjects of recommendation \cite{Abdollahpouri:2019:Recommendation}.  Finally, when goods are allocated, such as housing or subsidies one may need to observe both individual and group fairness  \cite{benabbou2018diversity}.  Indeed, group fairness is specifically important in, e.g., Singapore, which has specifically enforced notions of group fairness when allocating public housing \cite{Benabbou:2019:Housing}.

\noindent\textbf{Constrained reasoning in MAB.}  There is also significant recent work in constrained reasoning in the MAB setting.  
\citet{balakrishnan2018incorporating} study the idea of learning constraints over pulling arms by observation in a pre-training phase.
%
\citet{wu2015algorithms} study constraints in both number of pulls per arm, as well as number of rounds where arms are available to be pulled.
\citet{wu2016conservative} study a different flavor of constrained bandits where the learned policy cannot fall below a certain threshold; modeling the case where one wants to explore, but not suffer too much of a penalty over a status-quo policy.
A related and perhaps interesting direction for future work is the work on bandits that are budget-constrained (without fairness considerations).
\citet{ding2013multi} study budget-constrained bandits where each arm also has an unknown cost distribution and one must learn a policy that maximizes reward and minimizes cost.
Our formulation is not captured in the current literature on constrained and budgeted bandits and it is not obvious how to formalize a budget constraint as an inter-group fairness constraint. Indeed, a simplistic version of this would just lead to exhausting the budget of the ``better'' arm pulls before moving to the next best.

\noindent\textbf{Legal motivation.}  Fairness in bandits is a particularly important area as the online, dynamic nature makes the task challenging and the use of bandits in a number of areas makes the problem particularly relevant.  The motivating factor for group fairness is that one does not want to cause disparate impact, or the idea that groups should be treated differently based only on non-relevant aspects \cite{Feldman:15:DisparateImpact}.  Indeed, discrimination in certain areas including housing, credit, and jobs is forbidden in the US by the Civil Rights Act of 1965. It is specifically in these areas where bandit algorithms are deployed: advertising (where discrimination has been found) \cite{Sweeney:2013:AdDelivery},
college admissions~\cite{Schumann:2019:CohortSelection}, and interviewing~\cite{Schumann:2019:TieredInterviewing}. 


\section{Proofs}\label{ap:proofs}


\subsection{Two Groups}

In order to prove Theorem~\ref{thm:two_group_bound}, we first prove two lemmas.

\begin{lemma}\label{lemma:w}
The following holds for any $i$ at any time $t$, with probability at least $1-\delta$:
\begin{equation}\label{eq:w_inequality}
    \left| \hat{\beta}_{i,t} \cdot x_{i,t} - (\beta_i \cdot x_{i,t} +  \mathbbm{1}[i\in P_1]\psi_{P_1}\cdot x_{i,t}) \right|\leq w_{i,t}.
\end{equation}
\end{lemma}
\begin{proof}
There are two cases: $i\in P_1$ or $i\not\in P_1$.

Focusing on the first case, inequality~\ref{eq:w_inequality} becomes:
\begin{equation*}
    \left| \hat{\beta}_{i,t} \cdot x_{i,t} - \beta_i \cdot x_{i,t}\right|\leq w_{i,t}.
\end{equation*}
By the standard properties of OLS estimators~\cite{Kuan04:Classical},
\begin{equation*}
    \hat{\beta}{i,t}\sim\mathcal{N}\left(\beta_{i},\sigma^2(X_{i,t}^T,X_{i,t})^{-1}\right).
\end{equation*}
Then, for any fixed $x_{i,t}$:
\begin{equation*}
    \hat{\beta}_{i,t}\cdot x_{i,t}\sim\mathcal{N}\left(\beta_{i}\cdot x_{i,t},x_{i,t}^T\sigma^2(X_{i,t}^T,X_{i,t})^{-1}x_{i,t}\right).
\end{equation*}
Using the definition of the Quantile function and the symmetric property of the normal distribution, with probability at least $1-\frac{\delta}{nT}$,
\begin{equation*}
    \hat{\beta}_{i,t}\sim\mathcal{N}\left(\beta_{i},\sigma^2(X_{i,t}^T,X_{i,t})^{-1}\right).
\end{equation*}

Exploring the second case where $i\in P_1$, inequality~\ref{eq:w_inequality} can be replaced with 
\begin{equation*}
    \left| \hat{\beta}_{i,t} \cdot x_{i,t} - C_i \cdot x_{i,t}\right|\leq w_{i,t}
\end{equation*}
where $C_i=\beta_i+\psi_{P_1}$. Again, by the standard properties of OLS estimators $\hat{\beta}{i,t}\sim\mathcal{N}\left(C_{i},\sigma^2(X_{i,t}^T,X_{i,t})^{-1}\right)$, we have for any fixed $x_{i,t}$:
\begin{equation*}
    \hat{\beta}_{i,t}\cdot x_{i,t}\sim\mathcal{N}\left(C_{i}\cdot x_{i,t},x_{i,t}^T\sigma^2(X_{i,t}^T,X_{i,t})^{-1}x_{i,t}\right).
\end{equation*}
This uses the definition of the Quantile function and the symmetric property of the normal distribution, with probability at least $1-\frac{\delta}{nT}$.

Therefore, the probability that inequality~\ref{eq:w_inequality} fails to hold for any $i$ at any timestep $t$ is at most $nT\cdot\frac{\delta}{nT}=\delta$.
\end{proof}

\begin{lemma}\label{lemma:b}
The following holds for any group $P_j$, any arm $i$, at any time $t$, with probability at least $1-\delta$:
\begin{equation}\label{eq:psi_inequality}
\left| \hat{\psi}_{P_j,t} \cdot x_{i,t} - \bar{\psi}_{P_j} \cdot x_{i,t}\right|\leq b_{P_j,i,t}.
\end{equation}
\end{lemma}
\begin{proof}
By the standard properties of OLS estimators $\hat{\psi}_{P_j,t}\sim\mathcal{N}\left(\bar{\psi}_{P_j},\sigma^2(\mathcal{X}_{P_j,t}^T,\mathcal{X}_{P_j,t})^{-1}\right)$. For any fixed $x_{i,t}$,
\begin{equation*}
    \hat{\psi}_{i,t}\cdot x_{i,t}\sim\mathcal{N}\left(\bar{\psi}_{P_j}\cdot x_{i,t},x_{i,t}^T\sigma^2(\mathcal{X}_{P_j,t}^T,\mathcal{X}_{P_j,t})^{-1}x_{i,t}\right).
\end{equation*}
Using the definition of the quantile function and the symmetric property of the normal distribution, with probability at least $1-\frac{\delta}{\frac{n}{|P_j|}T}$, inequality~\ref{eq:psi_inequality} holds. Therefore, the probability that this fails to hold for any $i$ at any timestep $t$ is at most $\frac{n}{|P_j|}T\cdot \frac{\delta}{\frac{n}{|P_j|}T}=\delta$.
\end{proof}

With Lemma~\ref{lemma:w} and Lemma~\ref{lemma:b}, we can now prove Theorem~\ref{thm:two_group_bound}.
\begin{proof}
Regret for \textsc{GroupFairTopInterval} can be grouped into three terms for any $T_1\leq T$:
\begin{align}
    R^*(T) &= \sum_{t:\ t\textit{ is an explore round}} \mathit{regret}(t) \nonumber \\
    & + \sum_{t:\ t\textit{ is an exploit round and } t<T_1}\mathit{regret}(t) \nonumber \\
    & + \sum_{t:\ t\textit{ is an exploit round and }t\geq T_1}\mathit{regret}(t) \label{eq:split_regret}
\end{align}

Starting with the first term, define $p_t=\frac{1}{t^{1/3}}$ to be the probability that timestep $t$ is an exploration round. Then, for any $t$,
\begin{equation}\label{eq:p_t}
    \sum_{t'<t}p_{t'}=\Theta(t^{2/3}).
\end{equation}

We now focus on the third term of Equation~\ref{eq:split_regret}, where $t$ is an exploit round and $t>T_1$. Throughout the rest of the proof we assume Lemma~\ref{lemma:w} and Lemma~\ref{lemma:b}. Fix a exploit timestep $t$ where arm $i^t$ is played. Then,
\begin{align}
    \mathit{regret}(t) &\leq 2w_{i^t,t} + 2b_{P_1,i^t,t} + 2b_{P_2,i^t,t} \nonumber \\
    &\leq 2\max_i\left(w_{i,t} + b_{P_1,i,t} + b_{P_2,i,t}\right) \nonumber \\
    &\leq 2\left(\max_iw_{i,t} + \max_ib_{P_1,i,t} + \max_ib_{P_2,i,t}\right). \label{eq:2errror_bounds}
\end{align}

Note that:
\begin{equation*}
w_{i,t} = Q_{\mathcal{N}\left(0,x_{i,t}\left(X_{i,t}^TX_{i,t}\right)^{-1}x_{i,t}^T\right)}\left(\frac{\delta}{2nT}\right).
\end{equation*} 
Similarly, 
\begin{equation*}
b_{P_j,i,t} =Q_{\mathcal{N}\left(0,x_{i,t}\left(\mathcal{X}_{P_j,t}^T\mathcal{X}_{P_j,t}\right)^{-1}x_{i,t}^T\right)}\left(\frac{\delta}{2\frac{n}{\left|P_j\right|}T}\right).
\end{equation*}

We first bound
\begin{align}
    x_{i,t}\left(X_{i,t}^TX_{i,t}\right)^{-1}x_{i,t} &\leq ||x_{i,t}||\lambda_{\max} \left(\left(X_{i,t}^TX_{i,t}\right)^{-1}\right) \nonumber \\
    &=||x_{i,t}||\frac{1}{\lambda_{\min}\left(X_{i,t}^TX_{i,t}\right)} \nonumber \\
    &\leq\frac{1}{\lambda_{\min}\left(X_{i,t}^TX_{i,t}\right)}
\end{align}
where the last inequality holds since $||x_{i,t}||\leq 1$ for all $i$ and $t$. Using similar logic,
\begin{equation}
    x_{i,t}\left(\mathcal{X}_{P_j,t}^T\mathcal{X}_{P_j,t}\right)^{-1}x_{i,t}\leq \frac{1}{\lambda_{\min}\left(\mathcal{X}_{P_j,t}^T\mathcal{X}_{P_j,t}\right)}.
\end{equation}

Let $G_{i,t}$ be the number of observations of arm $i$ with contexts drawn uniformly from the distribution for arm $i$ prior to timestep $t$. Similarly, let $\mathcal{G}_{P_j,t}$ be the number of observations of group $P_j$ with contexts drawn uniformly from the distribution for group $P_j$ prior to timestep $t$. Let $L>\max_t\lambda_{\max}(x_{i,t}^T,x_{i,t})$. For any $\alpha\in[0,1]$, using the superaddivity of minimum eigenvectors for positive semidefinite matrices, we get
\begin{equation}\label{eq:G_i}
    \mathbb{E}\left[\lambda_{\min}(X_{i,t}^TX_{i,t})\right] \geq \frac{G_{i,t}}{d}\lambda_{\min_{i,d}}\geq \left\lfloor \frac{G_{i,t}}{d} \right\rfloor\lambda_{\min_{i,d}}.
\end{equation}
Similarly,
\begin{equation}\label{eq:G_P_j}
    \mathbb{E}\left[\lambda_{\min}(\mathcal{X}_{P_j,t}^T\mathcal{X}_{P_j,t})\right] \geq \frac{\mathcal{G}_{P_j,t}}{d}\lambda_{\min_{P_j,d}}\geq \left\lfloor \frac{\mathcal{G}_{P_j,t}}{d} \right\rfloor\lambda_{\min_{P_j,d}}.
\end{equation}
Equation~\ref{eq:G_i} implies that
\begin{align}
    &\ \Pr_{X_{i,t}}\left[\lambda_{\min}(X_{i,t}^T,X_{i,t}) \leq \alpha \left\lfloor\frac{G_{i,t}}{d}\right\rfloor\lambda_{\min_{i,d}}\right] \nonumber \\ 
    &\leq \Pr_{X_{i,t}}\left[\lambda_{\min}(X_{i,t}^T,X_{i,t}) \leq \alpha \mathbb{E}[\lambda_{\min}(X_{i,t}^TX_{i,t})]\right] \label{eq:lambda_min_1} \\
    &\leq \Pr_{X_{i,t}}\left[\lambda_{\min}(X_{i,t}^T,X_{i,t}) \leq \alpha \lambda_{\min}(\mathbb{E}[X_{i,t}^TX_{i,t}])\right] \label{eq:lambda_min_2} \\
    &\leq d\exp\left(\frac{-(1-\alpha)^2\lambda_{\min}(\mathbb{E}[X_{i,t}^TX_{i,t}])}{2L}\right) \label{eq:lambda_min_3} \\
    &\leq d\exp\left(\frac{-(1-\alpha)^2\mathbb{E}[\lambda_{\min}(X_{i,t}^TX_{i,t})]}{2L}\right) \label{eq:lambda_min_4} \\
    &\leq d\exp\left(\frac{-(1-\alpha)^2\left\lfloor\frac{G_{i,t}}{d}\right\rfloor\lambda_{\min_{i,d}}}{2L}\right) \label{eq:lambda_min_final} 
\end{align}
where Inequalities~\ref{eq:lambda_min_1} and \ref{eq:lambda_min_final} are from equation~\ref{eq:G_i}, Inequalities~\ref{eq:lambda_min_2} and \ref{eq:lambda_min_4} are from Jensen's inequality \cite{mitzenmacher2017probability}, and Inequality~\ref{eq:lambda_min_3} uses a Matrix Chernoff Bound \cite{mitzenmacher2017probability}. 

Using Inequality~\ref{eq:lambda_min_final} after rearranging with probability $1-\delta$:
\begin{equation}\label{eq:lambda_i}
    \lambda_{\min}(X_{i,t}^TX_{i,t}) \geq \alpha\left\lfloor\frac{G_{i,t}}{d}\right\rfloor \lambda_{\min_{i,d}}
\end{equation}
when
\begin{equation}
    G_{i,t}\geq d\left(\frac{L}{(1-\alpha)^2\lambda_{\min_{i,d}}}\right)\left(\ln\frac{1}{\delta}+\ln d\right).
\end{equation}
Using similar logic with probability $1-\delta$, we have
\begin{equation}\label{eq:lambda_P_j}
    \lambda_{\min}(\mathcal{X}_{P_j,t}^T\mathcal{X}_{P_j,t})\geq\alpha\left\lfloor\frac{\mathcal{G}_{P_j,t}}{d}\right\rfloor\lambda_{\min_{P_j,d}}
\end{equation}
when
\begin{equation}
    \mathcal{G}_{P_j,t}\geq d\left(\frac{L}{(1-\alpha)^2\lambda_{\min_{P_j,d}}}\right)\left(\ln\frac{1}{\delta}+\ln d\right).
\end{equation}

Using a multiplicative Chernoff bound \cite{mitzenmacher2017probability} for a fixed timestep $t$ with probability $1-\delta'$, the number of exploitation rounds prior to rounds $t$ will satisfy
\begin{equation}\label{eq:G_exploit}
    \left|G_t-\sum_{t'<t}p_{t'}\right| \leq \sqrt{\ln\frac{2}{\delta'}\sum_{t<t'}p_{t'}}
\end{equation}
For a fixed $i$ and timestep $t$ using a multiplicative Chernoff bound, with probability $1-\delta'$, the number of exploitation rounds for arm $i$ prior to round $t$ will satisfy
\begin{equation}\label{eq:G_i_exploit}
    \left|G_{i,t}-\frac{G_t}{n}\right| \leq \sqrt{\ln\frac{2}{\delta'}\frac{G_t}{n}}.
\end{equation}
Similarly, for a fixed group $P_j$ and timestep $t$ with probaility $1-\delta'$, the number of exploration rounds for group $P_j$ prior to round $t$ will satisfy
\begin{equation}\label{eq:G_P_j_exploit}
    \left|\mathcal{G}_{i,t}-\frac{G_t}{|P_j|/n}\right| \leq \sqrt{\ln\frac{2}{\delta'}\frac{G_t}{n/|P_j|}}
\end{equation}
where $|P_j|$ is the size of group $P_j$.

Combining equations~\ref{eq:G_exploit} and \ref{eq:G_i_exploit} with probability at least $1-2\delta'$ for a fixed arm $i$ and timestep $t$, if $\sum_{t'<t}P_{t'}\geq 36n\ln^2\frac{2}{\delta'}$ we have
\begin{equation}\label{eq:G_i_exploit_final}
    \left|G_{i,t}-\frac{\sum_{t'<t}p_{t'}}{n}\right| \leq \frac{\sum_{t'<t}p_{t'}}{2n}.
\end{equation}
Similarly, combining equations~\ref{eq:G_exploit} and \ref{eq:G_P_j_exploit} with probability at least $1-2\delta'$ for a fixed group $P_j$ and timestep $t$:
\begin{equation}\label{eq:G_P_j_exploit_final}
    \left|\mathcal{G}_{i,t}-\frac{\sum_{t'<t}P_{t'}}{n/|P_j|}\right| \leq \frac{\sum_{t'<t}p_{t'}}{2n}.
\end{equation}

Therefore, equation~\ref{eq:lambda_i} holds with probability $1-\delta'$ when
\begin{equation}
    \frac{\sum_{t'<t}p_t}{2n} \geq d\left(\frac{L}{(1-\alpha)^2\lambda_{\min_{i,d}}}\right)\left(\ln\frac{1}{\delta}+ \ln d\right).
\end{equation}
Similarly, equation~\ref{eq:lambda_P_j} holds with probability $1-\delta'$ when
\begin{equation}
    \frac{\sum_{t'<t}p_t}{2n/|P_j|} \geq d\left(\frac{L}{(1-\alpha)^2\lambda_{\min_{P_j,d}}}\right)\left(\ln\frac{1}{\delta}+ \ln d\right).
\end{equation}

Therefore, since $n/|P_j|<n$, the number of rounds after which we have sufficient samples such that the estimators are well-concentrated is
\begin{equation}\label{eq:T_1}
    T_1=\Theta\left(\min_{a}\left(\frac{dnL}{\lambda_{min_a,d}}\left(\ln^2\frac{2}{\delta} + \ln d\right)\right)^{3/2}\right)
\end{equation}
where $a\in[n]\cup P_1 \cup P_2$.

Also note that for any $t\geq T_1$ we have
\begin{equation}
    \sum_{t'<t} p_{t'}=\Omega\left(\min_a\left(\frac{dnL}{\lambda_{\min_{a,d}}}\left(\ln^2\frac{2}{\delta'} + \ln d\right)\right)\right).
\end{equation}

We can now bound the third term in Equation~\ref{eq:split_regret}.
\begin{align}
    &\sum_{t:\ t\textit{ is an exploit round and }t\geq T_1}regret(t) \nonumber \\
    &\leq 2\sum_{t\geq T_1}\left(\max_i w_{i,t} + \max_i b_{P_1,i,t} + \max_i b_{P_2,i,t}\right) \label{eq:big1} \\
    &\leq 2\sum_{t\geq T_1}\left(\max_i Q_{\mathcal{N}\left(0,\lambda_{\max}((X_{i,t}^TX_{i,t}))^{-1}\right)}\left(\frac{\delta}{2nT}\right) \right. \nonumber \\
    &\ \ \ \ + \max_i Q_{\mathcal{N}\left(0,\lambda_{\max}((\mathcal{X}_{P_1,t}^T\mathcal{X}_{P_1,t}))^{-1}\right)}\left(\frac{\delta}{2\frac{n}{|P_1|}T}\right) \nonumber \\ 
    &\ \ \ \ \left. + \max_i Q_{\mathcal{N}\left(0,\lambda_{\max}((\mathcal{X}_{P_2,t}^T\mathcal{X}_{P_2,t}))^{-1}\right)}\left(\frac{\delta}{2\frac{n}{|P_2|}T}\right)\right) \nonumber \\
    &\leq 2\sum_{t\geq T_1}\left( Q_{\mathcal{N}\left(0,\frac{1}{\min_i\lambda_{\min}((X_{i,t}^TX_{i,t}))^{-1}}\right)}\left(\frac{\delta}{2nT}\right) \right. \nonumber \\
    &\ \ \ \ + Q_{\mathcal{N}\left(0,\frac{1}{\min_i\lambda_{\min}((\mathcal{X}_{P_1,t}^T\mathcal{X}_{P_1,t}))^{-1}}\right)}\left(\frac{\delta}{2\frac{n}{|P_1|}T}\right) \nonumber \\ 
    &\ \ \ \ \left. +\  Q_{\mathcal{N}\left(0,\frac{1}{\min_i\lambda_{\min}((\mathcal{X}_{P_2,t}^T\mathcal{X}_{P_2,t}))^{-1}}\right)}\left(\frac{\delta}{2\frac{n}{|P_2|}T}\right)\right) \nonumber\\
    &\leq 2\sum_{t\geq T_1}\left( Q_{\mathcal{N}\left(0,\frac{1}{\min_i\alpha\left\lfloor \frac{G_{i,t}}{d}\right\rfloor\lambda_{\min_{i,d}}}\right)}\left(\frac{\delta}{2nT}\right) \right. \nonumber \\
    &\ \ \ \ + Q_{\mathcal{N}\left(0,\frac{1}{\alpha\left\lfloor\frac{\mathcal{G}_{P_1,t}}{d}\right\rfloor\lambda_{\min_{P_1,d}}}\right)}\left(\frac{\delta}{2\frac{n}{|P_1|}T}\right) \nonumber \\ 
    &\ \ \ \ \left. +\  Q_{\mathcal{N}\left(0,\frac{1}{\alpha\left\lfloor\frac{\mathcal{G}_{P_2,t}}{d}\right\rfloor\lambda_{\min_{P_2,d}}}\right)}\left(\frac{\delta}{2\frac{n}{|P_2|}T}\right)\right) + 3\delta'T \label{eq:big2} \\
    &\leq 2\sum_{t\geq T_1}\left( \sqrt{\frac{\ln\frac{2nT}{\delta}}{\min_i\alpha\left\lfloor\frac{G_{i,t}}{d}\right\rfloor\lambda_{\min_{i,d}}}} \right. \nonumber \\
    &\ \ \ \ + \sqrt{\frac{\ln\frac{2\frac{n}{|P_1|}T}{\delta}}{\min_i\alpha\left\lfloor\frac{\mathcal{G}_{P_1,t}}{d}\right\rfloor\lambda_{\min_{P_1,d}}}} \nonumber \\
    &\ \ \ \ \left. + \sqrt{\frac{\ln\frac{2\frac{n}{|P_2|}T}{\delta}}{\min_i\alpha\left\lfloor\frac{\mathcal{G}_{P_2,t}}{d}\right\rfloor\lambda_{\min_{P_2,d}}}} \right) + 6\delta'T \label{eq:big3} \\
    &\leq 2\sum_{t\geq T_1}\left(3\sqrt{\frac{\ln\frac{2nT}{\delta}}{\min_i\alpha\left\lfloor\frac{G_{i,t}}{d}\right\rfloor\lambda_{\min_{i,d}}}}\right) \label{eq:big4} \\
    &= O\left(\sum_{t\geq T_1}\sqrt{d\frac{\ln\frac{2nT}{\delta}}{\min_iG_{i,t}\lambda_{\min_{i,d}}}} + \delta'T\right) \nonumber \\
    &= O\left(\sqrt{d\frac{\ln\frac{2nT}{\delta}}{\min_i\lambda_{\min_{i,d}}}}\sum_{t\geq T_1}\sqrt{\frac{1}{\min_iG_{i,t}}} + \delta'T\right) \nonumber \\
    &= O\left(\sqrt{d\frac{\ln\frac{2nT}{\delta}}{\min_i\lambda_{\min_{i,d}}}}\sum_{t\geq T_1}\sqrt{\frac{n}{\sum_{t'<t}p_{t'}}} + \delta'T\right) \nonumber \\
    &= O\left(\sqrt{d\frac{\ln\frac{2nT}{\delta}}{\min_i\lambda_{\min_{i,d}}}}\sum_{t\geq T_1}\sqrt{\frac{n}{t^{2/3}}} + \delta'T\right) \label{eq:big5} \\
    &= O\left(\sqrt{dn\frac{\ln\frac{2nT}{\delta}}{\min_i\lambda_{\min_{i,d}}}}\sum_{t\in [T_1,T]}\frac{1}{t^{1/3}} + \delta'T\right) \nonumber \\
    &= O\left(\sqrt{dn\frac{\ln\frac{2nT}{\delta}}{\min_i\lambda_{\min_{i,d}}}}T^{2/3} + \delta'T\right) \label{eq:big_final}
\end{align}
where (\ref{eq:big1}) is due to Equation~\ref{eq:2errror_bounds}, (\ref{eq:big2}) is due to Equations~\ref{eq:G_i} and \ref{eq:G_P_j}, (\ref{eq:big3}) is due to Chernoff bounds, (\ref{eq:big4}) is due to the fact that $\frac{n}{|P_j|}<n$ and $G_{P_j,t}>\min_i G_{i,t}$, and (\ref{eq:big5}) is due to Equation~\ref{eq:p_t}.
Theorem~\ref{thm:two_group_bound} follows by combining Equations~\ref{eq:split_regret}, \ref{eq:p_t}, \ref{eq:T_1}, and \ref{eq:big_final} and setting $\delta'=\min\left(\frac{1}{3nT},\frac{1}{T^{1/3}}\right)$.

\end{proof}

\subsection{Multiple Groups}\label{app:multiple_groups}

\begin{algorithm}
\caption{\textsc{GroupFairTopInterval (Multiple Groups)}}\label{alg:fair_top_interval_mult_group}
\begin{algorithmic}[1]
\REQUIRE $\delta$, ($P_1,\ldots,P_m$), $\rho$
\FOR{$t=1\ldots T$}
\STATE with probability $\frac{1}{t^{1/3}}$, play $i_t\in_R\{1,\ldots , n\}$
\STATE \textbf{Else}
\begin{ALC@g}
\FOR{$j=1\ldots,m$}
\STATE Let $\hat{\psi}_{P_j,t} = \left(\mathcal{X}_{P_j,t}^T\mathcal{X}_{P_j,t}\right)^{-1}\mathcal{X}_{P_j,t}^T\mathcal{Y}_{P_j,t}$
\ENDFOR
\FOR{$i=1\ldots n$}
\STATE Let $\hat{\beta}_{i,t}=\left(X_{i,t}^TX_{i,t}\right)^{-1}X_{i,t}^TY_{i,t}^T$ \label{line:beta_multi}
\STATE Let $F_{i,t}=\mathcal{N}\left(0,\sigma^2x_{i,t}\left(X_{i,t}^TX_{i,t}\right)^{-1}x_{i,t}^T\right)$
\STATE Let $w_{i,t}=Q_{F_{i,t}}\left(\frac{\delta}{2nt}\right)$ \label{line:w_mult}
\FOR{$j$ where $i\in P_j$}
\STATE Let $\mathcal{F}_{P_j,i,t}=\mathcal{N}\left(0,\sigma^2x_{i,t}\left(\mathcal{X}_{P_j,t}^T\mathcal{X}_{P_j,t}\right)x_{i,t}^T\right)$
\STATE Let $b_{P_j,i,t}=Q_{\mathcal{F}_{P_j,i,t}}\left(\frac{\delta}{2\frac{n}{|P_j|}T}\right)$\label{line:b_j}
\STATE Let $\hat{u}_{i,t}=\hat{\beta}_{i,t}\cdot x_{i,t} + w_{i,t} + \rho - \hat{\psi}_{P_j,t} \cdot x_{i,t} + b_{P_j,i,t}$\label{line:upper_j}
\ENDFOR
\ENDFOR
\STATE Play $\argmax_i\hat{u}_{i,t}$ and observe reward $y_{i,t}$\label{line:pull_mult}
\end{ALC@g}
\ENDFOR
\end{algorithmic}
\end{algorithm}

In in order to prove Theorem~\ref{thm:mult_group_bound}, we first prove two lemmas.
\begin{lemma}\label{lemma:w_mult}
The following holds for any $i$ at any time $t$, with probability at least $1-\delta$
\begin{align}
    &\left|\hat{\beta}_{i,t}\cdot x_{i,t} - \left(\beta_i\cdot x_{i,t} + \sum_{j=1}^m \mathbbm{1}
    \left[i\in P_j\right]\psi_{P_j}\cdot x_{i,t}\right)\right| \nonumber \\
    &\leq w_{i,t} \label{eq:lem_w_mult}
\end{align}
\end{lemma}

\begin{proof}
Inequality~\ref{eq:lem_w_mult} can be replaced with
\begin{equation*}
    \left|\hat{\beta}_{i,t}\cdot x_{i,t} - C_i\cdot x_{i,t}\right| \leq w_{i,t}
\end{equation*}
where $C_i=\beta_i+\psi_{P_j}$ and $i\in P_j$. By the standard properties of OLS estimators $\hat{\beta}_{i,t}\sim N\left(C_i, \sigma^2(X_{i,t}^TX_{i,t})^{-1}\right)$. For any fixed $x_{i,t}$:
\begin{equation*}
    \hat{\beta}_{i,t}\cdot x_{i,t}\sim N\left(C_i\cdot x_{i,t}, x_{i,t}^T\sigma^2(X_{i,t}^TX_{i,t})^{-1}x_{i,t}\right)
\end{equation*}
Using the definition of the quantile function and the symmetric property of the normal distribution, with probability at least $1-\frac{\delta}{nT}$, Inequality~\ref{eq:lem_w_mult} holds. Therefore, the probability that inequality~\ref{eq:lem_w_mult} fails to hold for any $i$ at any timestep $t$ is at most $nT\frac{\delta}{nT}=\delta$.
\end{proof}

\begin{lemma}\label{lemma:b_mult}
The following holds for any group $P_j$, any arm $i$, at any timestep $t$, with probability at least $1-\delta$:
\begin{equation}
    \left|\hat{\psi}_{P_j,t}\cdot x_{i,t} - \psi_{P_j,t}\cdot x_{i,t}\right| \leq b_{P_j,i,t}.\label{eq:lem_b_mult}
\end{equation}
\end{lemma}
\begin{proof}
By the standard properties of OLS estimators,
\begin{equation*}
    \hat{\psi}_{P_j,t}\sim N\left(\psi_{P_j}, \sigma^2(\mathcal{X}_{P_j, t}^T\mathcal{X}_{P_j,t})^{-1}\right).
\end{equation*}
For any fixed $x_{i,t}$:
\begin{equation*}
    \hat{\psi}_{P_j,t}\cdot x_{i,t} \sim N\left(\psi_{P_j}\cdot x_{i,t}, x_{i,t}^T\sigma^2(\mathcal{X}_{P_j,t}^T\mathcal{X}_{P_j,t})^{-1}x_{i,t}\right).
\end{equation*}
Using the definition of the quantile function and the symmetric property of the normal distribution, with probability of at least $1-\frac{\delta}{\frac{n}{|P_j|}T}$ inequality~\ref{eq:lem_b_mult} holds. Therefore the probability this fails to hold for any $i$ at timestep $t$ is at most $\frac{n}{|P_j|}T\frac{\delta}{\frac{n}{|P_j|}T}=\delta$.
\end{proof}

\begin{theorem}\label{thm:mult_group_bound}
For $m$ groups $P_1,\ldots,P_m$, where $\rho$ is the expected average reward, \textsc{GroupFairTopInterval (Multiple Groups)} has regret
\begin{align}
    R^*(T) &= O\left( \sqrt{\frac{dn\ln\frac{2nT}{\delta}}{l}}T^{2/3}  \right. \nonumber \\
    & \left. + \left( \frac{dnmL}{l}\left(\ln^2\frac{2nT}{\delta}+\ln d\right)\right)^{2/3} \right)
\end{align} 
where $l=\min_i\lambda_{min_{i,d}}$ and $L>\max_t\lambda_{\max}(x_{i,t}^Tx_{i,t})$.
\end{theorem}
We can now prove Theorem~\ref{thm:mult_group_bound}.
\begin{proof}
Assume that both Lemma~\ref{lemma:w_mult} and Lemma~\ref{lemma:b_mult} hold for all arms $i$ and all timesteps $t$.

Regret for \textsc{GroupFairTopInterval (Multiple Groups)} can be grouped into three terms for any $T_1\leq T$:
\begin{align}
    R^*(T) &= \sum_{t:\ t\text{ is an explore round}} \mathit{regret}(t) \nonumber \\
    & + \sum_{t:\ t\text{ is an exploit round and } t<T_1} \mathit{regret}(t) \nonumber \\
    & + \sum_{t:\ t\text{ is an exploit round and } t\geq T_1} \mathit{regret}(t) \label{eq:mult_regret}
\end{align}

Starting with the first term in Equation~\ref{eq:mult_regret}, define $p_t=\frac{1}{t^{1/3}}$ to be the probability that timestep $t$ is an exploration round. Then, for any $t$,
\begin{equation}
    \sum_{t'<t}p_{t'}=\Theta(t^{2/3})\label{eq:mult_p}
\end{equation}

Focusing on the third term of Equation~\ref{eq:mult_regret}, fix an exploit timestep $t$ where arm $i_t$ is played. Then,
\begin{align}
    \mathit{regret}(t)&\leq 2w_{i_t,t} + \max_j(2b_{P_j,i_t,t}) \nonumber \\
    &\leq 2\max_{i,j}(w_{i,t}+b_{P_j,i,t}) \nonumber \\
    &\leq 2\left(\max_iw_{i,t} + \max_{i,j}b_{P_j,i,t}\right)\label{eq:mult_1}
\end{align}

From Algorithm~\ref{alg:fair_top_interval_mult_group}, note that
\begin{equation*}
    w_{i,t}=Q_{N\left(0,x_{i,t}\left(X_{i,t}^TX_{i,t}\right)^{-1}x_{i,t}^T\right)}\left(\frac{\delta}{2nT}\right).
\end{equation*}
Similarly,
\begin{equation*}
    b_{P_j,i,t}=Q_{N\left(0,x_{i,t}\left(\mathcal{X}_{P_j,t}^T\mathcal{X}_{P_j,t}\right)^{-1}x_{i,t}^T\right)}\left(\frac{\delta}{2\frac{n}{|P_j|}T}\right)
\end{equation*}

We will first bound $x_{i,t}\left(X_{i,t}^TX_{i,t}\right)^{-1}x_{i,t}^T$.
\begin{align}
    x_{i,t}\left(X_{i,t}^TX_{i,t}\right)^{-1}x_{i,t}^T &\leq ||x_{i,t}||\lambda_{\max}\left(\left(X_{i,t}^TX_{i,t}\right)^{-1}\right) \nonumber \\
    &=||x_{i,t}||\frac{1}{\lambda_{\min}(X_{i,t}^TX_{i,t})} \nonumber \\
    &\leq \frac{1}{\lambda_{\min}(X_{i,t}^TX_{i,t})}\label{eq:mult_2}
\end{align}
where inequality~\ref{eq:mult_2} is due to $||x_{i,t}||\leq 1$ for all arms $i$ and all timesteps $t$.

Using similar logic:
\begin{equation}
    x_{i,t}\left(\mathcal{X}_{P_j,t}^T\mathcal{X}_{P_j,t}\right)^{-1}x_{i,t}^T \leq \frac{1}{\lambda_{\min}\left(\mathcal{X}_{P_j,t}^T\mathcal{X}_{P_j,t}\right)}.\label{eq:mult_3}
\end{equation}

Let $G_{i,t}$ be the number of observations of arm $i$ with context $i$ drawn uniformly from the distribution for arm $i$ prior to timestep $t$. Similarly, let $\mathcal{G}_{P_j,t}$ be the number of observations of group $P_j$ with context drawn uniformly from the distribution for group $P_j$ prior to timestep $t$. Let $L> \max_t\lambda_{\max}\left(x_{i,t}^Tx_{i,t}\right)$.

For any $\alpha\in[0,1]$, using the superadditivity of minimum eugenvectors for positive semi-definite matrices, we get:
\begin{align}
    \mathbb{E}\left[\lambda_{\min}(X_{i,t}^TX_{i,t})\right]&\geq\frac{G_{i,t}}{d}\lambda_{\min_{i,d}} \nonumber \\
    &\geq \left\lfloor\frac{G_{i,t}}{d}\right\rfloor.\label{eq:mult_4}
\end{align}
Similarly,
\begin{align}
    \mathbb{E}\left[\lambda_{min}(\mathcal{X}_{P_j,t}^T\mathcal{X}_{P_j,t})\right]&\geq \left\lfloor\frac{G_{P_j,t}}{d}\right\rfloor\lambda_{\min_{P_j.d}}.\label{eq:mult_5}
\end{align}

Equation~\ref{eq:mult_4} implies that:
\begin{align}
    &\Pr_{x_{i,t}}\left[\lambda_{\min}(X_{i,t}X_{i,t})\leq \alpha\left\lfloor\frac{G_{i,t}}{d}\right\rfloor\lambda_{\min_{i,d}}\right] \nonumber \\
    &\leq \Pr_{x_{i,t}}\left[\lambda_{\min}(X_{i,t}^TX_{i,t})\leq \alpha\mathbb{E}\left[\lambda_{\min}(X_{i,t}^TX_{i,t})\right]\right] \label{eq:mult_6.1} \\
    &\leq \Pr_{x_{i,t}}\left[\lambda_{\min}(X_{i,t}^TX_{i,t})\leq \alpha\lambda_{\min}\left(\mathbb{E}\left[X_{i,t}^TX_{i,t}\right]\right)\right]\label{eq:mult_6.2} \\
    &\leq d\exp\left(\frac{-(1-\alpha)^2\lambda_{\min}\left(\mathbb{E}\left[X_{i,t}^TX_{i,t}\right]\right)}{2L}\right) \label{eq:mult_6.3} \\
    &\leq d\exp\left(\frac{-(1-\alpha)^2\mathbb{E}\left[\lambda_{\min}\left(X_{i,t}^TX_{i,t}\right)\right]}{2L}\right) \label{eq:mult_6.4} \\
    &\leq d\exp\left(\frac{-(1-\alpha)^2\left\lfloor\frac{G_{i,t}}{d}\right\rfloor\lambda_{\min_{i,d}}}{2L}\right) \label{eq:mult_6.final}
\end{align}
where inequality~\ref{eq:mult_6.1} comes from inequality~\ref{eq:mult_4}, inequality~\ref{eq:mult_6.2} is due to Jensen's inequality, inequality~\ref{eq:mult_6.3} is due to a matrix Chernoff Bound, inequality~\ref{eq:mult_6.4} is due to Jensen's inequality, and inequality~\ref{eq:mult_6.final} is due to inequality~\ref{eq:mult_4}. After rearranging inequality~\ref{eq:mult_6.final}, with probability $1-\delta$, 
\begin{equation}
    \lambda_{min}(X_{i,t}^TX_{i,t}) \geq \alpha \left\lfloor \frac{G_{i,t}}{d} \right\rfloor \lambda_{\min_{i,d}} \label{eq:mult_7}
\end{equation}
when
\begin{equation}
    G_{i,t}\geq d\left(\frac{L}{(1-\alpha)^2\lambda_{\min_{i,d}}}\right)\left(\ln{\frac{1}{\delta}} + \ln{d}\right). \label{eq:multi_8}
\end{equation}

Using similar logic with probability $1-\delta$, we have
\begin{equation}
    \lambda_{\min}\left(\mathcal{X}_{P_j,t}^T\mathcal{X}_{P_j,t}\right) \geq \alpha\left\lfloor\frac{G_{P_j,t}}{d}\right\rfloor\lambda_{\min_{P_j,d}} \label{eq:mult_9}
\end{equation}
when
\begin{equation}
    \mathcal{G}_{P_j,t} \geq d\left(\frac{L}{(1-\alpha)^2\lambda_{\min_{P_j,d}}}\right)\left(\ln\frac{1}{\delta}+\ln d\right).\label{eq:mult_10}
\end{equation}

Using a multiplicative Chernoff bound for a fixed timestep $t$ with probability $1-\delta'$, the number of exploitation rounds prior to rount $t$ will satisfy
\begin{equation}
    \left|G_t-\sum_{t'<t}p_{t'}\right| \leq \sqrt{\ln\frac{2}{\delta'}\sum_{t'<t}p_{t'}}.\label{eq:mult_11}
\end{equation}
For a fixed $i$ and timestep $t$, using a multiplicative Chernoff bound for a fixed timestep $t$ with probability $1-\delta'$, the number of exploitation rounds for arm $i$ prior to round $t$ will satisfy
\begin{equation}
    \left|G_{i,t}-\frac{G_t}{n}\right| \leq \sqrt{\ln\frac{2}{\delta'}\frac{G_t}{n}}\label{eq:mult_12}
\end{equation}

Similarly, for a fixed group $P_j$ and timestep $t$ with probability $1-\delta'$, the number of exploration rounds for group $P_j$ prior to round $t$ will satisfy
\begin{equation}
    \left|G_{P_j,t}-\frac{G_t}{n/|P_j|}\right| \leq \sqrt{\ln\frac{2}{\delta'}\frac{G_t}{n/|P_j|}}\label{eq:mult_13}
\end{equation}
where $|P_j|$ is the size of group $P_j$.

Combining inequality~\ref{eq:mult_11} and inequality~\ref{eq:mult_12}, with probability $1-2\delta'$ for a fixed arm $i$ and timestep $t$, if $\sum_{t'<t}p_{t'} \geq 36n\ln^2\frac{2}{\delta'}$ we have
\begin{equation}
    \left|G_{i,t}-\frac{\sum_{t'<t}p_{t'}}{n}\right|\leq \frac{\sum_{t'<t}p_{t'}}{2n}.\label{eq:mult_14}
\end{equation}

Similarly, combining inequality~\ref{eq:mult_11} and inequality~\ref{eq:mult_13} with probability at least $1-2\delta'$ for a fixed group $P_j$ and fixed timestep $t$:
\begin{equation}
    \left|G_{i,t}-\frac{\sum_{t'<t}p_{t'}}{n/|P_j|}\right|\leq \frac{\sum_{t'<t}p_{t'}}{2n/|P_j|}.\label{eq:mult_15}
\end{equation}

Therefore inequality~\ref{eq:mult_7} holds with probability $1-\delta'$ when
\begin{equation}
    \frac{\sum_{t'<t}p_{t'}}{2n}\geq d\left(\frac{L}{(1-\alpha)^2\lambda_{\min_{i,d}}}\right)\left(\ln\frac{1}{\delta} + \ln d\right). \label{eq:mult_16}
\end{equation}

Similarly, inequality~\ref{eq:mult_9} holds with probability $1-\delta'$ when
\begin{equation}
    \frac{\sum_{t'<t}p_{t'}}{2n/|P_j|} \geq d \left(\frac{L}{(1-\alpha)^2\lambda_{\min_{i,d}}}\right)\left(\ln\frac{1}{\delta}+\ln d\right). \label{eq:mult_17}
\end{equation}

Therefore, since $\frac{n}{|P_j|}<n$, the number of rounds after which we have sufficient samples such that the estimators are well-concentrated is
\begin{equation}
    T_1=\Theta\left(\min_a\left(\frac{dnmL}{\lambda_{\min_{a,d}}}\left(\ln^2\frac{2}{\delta}+\ln d\right)\right)^{3/2}\right) \label{eq:mult_T1}
\end{equation}
where $a\in[n]\cup P_1\cup\cdots\cup P_m$.

Also note that for any $t>T_1$ we have:
\begin{equation}
    \sum_{t'<t}p_{t'} = \Omega\left(\min_a\left(\frac{dnmL}{2\min_{a,d}}\left(\ln^2\frac{2}{\delta'} + \ln d\right)\right)\right). \label{eq:mult_18}
\end{equation}

Now we can bound the third term in equation~\ref{eq:mult_regret}.
\begin{align}
    &\sum_{t:\ t\text{ is an exploit round and } t>T_1} \mathit{regret}(t)\nonumber \\
    &\leq 2\sum_{t>T_1} \left(\max_iw_{i,t} + \max_{i,j}b_{P_j,i,t}\right) \label{eq:multi_19.1} \\
    &\leq 2\sum_{t>T_1} \left(\max_i Q_{N\left(0,\lambda_{\max}(X_{i,t}^TX_{i,t})^{-1}\right)}\left(\frac{\delta}{2nT}\right) \right. \nonumber \\
    &\ \left. + \max_j Q_{N\left(0,\lambda_{\max}(\mathcal{X}_{P_j,t}^T\mathcal{X}_{P_j,t})^{-1}\right)}\left(\frac{\delta}{2\frac{n}{|P_j|}T}\right)\right) \nonumber \\
    &\leq 2\sum_{t>T_1} \left( Q_{N\left(0,\frac{1}{\min_i\lambda_{\min}(X_{i,t}^TX_{i,t})}\right)}\left(\frac{\delta}{2nT}\right) \right. \nonumber \\
    &\ \left. +\  Q_{N\left(0,\frac{1}{\min_j\lambda_{\min}(\mathcal{X}_{P_j,t}^T\mathcal{X}_{P_j,t})}\right)}\left(\frac{\delta}{2\frac{n}{|P_j|}T}\right)\right) \nonumber \\
    &\leq 2\sum_{t>T_1} \left( Q_{N\left(0,\frac{1}{\min_i\alpha\left\lfloor\frac{G_{i,t}}{d}\right\rfloor\lambda_{\min_{i,d}}}\right)}\left(\frac{\delta}{2nT}\right) \right. \nonumber \\
    &\ \left. +\  Q_{N\left(0,\frac{1}{\min_j\alpha\left\lfloor\frac{\mathcal{G}_{P_j,t}}{d}\right\rfloor\lambda_{\min_{P_j,d}}}\right)}\left(\frac{\delta}{2\frac{n}{|P_j|}T}\right)\right) + 3\delta'T\label{eq:multi_19.2} \\
    &\leq 2\sum_{t>T_1} \left(\sqrt{\frac{\ln\frac{2nT}{\delta}}{\min_i\alpha\left\lfloor\frac{G_{i,t}}{d}\right\rfloor\lambda_{\min_{i,d}}}}\right. \nonumber \\
    &\ \left. + \sqrt{\frac{\ln\frac{2\frac{n}{min_j|P_j|}T}{\delta}}{\min_i\alpha\left\lfloor\frac{\mathcal{G}_{P_j,t}}{d}\right\rfloor\lambda_{\min_{i,d}}}}\right) + 6\delta'T \label{eq:multi_19.3} \\
    &\leq 2\sum_{t>T_1} \left(2\sqrt{\frac{\ln\frac{2nT}{\delta}}{\min_i\alpha\left\lfloor\frac{G_{i,t}}{d}\right\rfloor\lambda_{\min_{i,d}}}}\right) +6\delta'T \label{eq:multi_19.4} \\
    &=O\left(\sum_{t>T_1}\sqrt{d\frac{\ln\frac{2nT}{\delta}}{\min_iG_{i,t}\lambda_{\min_{i,d}}}} + \delta'T\right) \nonumber \\
    &=O\left(\sqrt{d\frac{\ln\frac{2nT}{\delta}}{\min_i\lambda_{\min_{i,d}}}}\sum_{t>T_1}\sqrt{\frac{1}{\min_iG_{i,t}}} + \delta'T\right) \nonumber \\
    &=O\left(\sqrt{d\frac{\ln\frac{2nT}{\delta}}{\min_i\lambda_{\min_{i,d}}}}\sum_{t>T_1}\sqrt{\frac{n}{\sum_{t'<t}p_{t'}}} + \delta'T\right) \nonumber \\
    &=O\left(\sqrt{d\frac{\ln\frac{2nT}{\delta}}{\min_i\lambda_{\min_{i,d}}}}\sum_{t>T_1}\sqrt{\frac{n}{t^{2/3}}} + \delta'T\right) \label{eq:multi_19.5} \\
    &=O\left(\sqrt{dn\frac{\ln\frac{2nT}{\delta}}{\min_i\lambda_{\min_{i,d}}}}\sum_{t\in[T_1,T]}\frac{1}{t^{1/3}} + \delta'T\right) \nonumber \\
    &=O\left(\sqrt{dn\frac{\ln\frac{2nT}{\delta}}{\min_i\lambda_{\min_{i,d}}}}T^{2/3} + \delta'T\right) \label{eq:multi_19.final}
\end{align}
where inequality~\ref{eq:multi_19.1} is due to equation~\ref{eq:mult_1}, inequality~\ref{eq:multi_19.2} is due to equation~\ref{eq:mult_4} and equation~\ref{eq:mult_5}, inequality~\ref{eq:multi_19.3} is due to a Chernoff bound, inequality~\ref{eq:multi_19.4} is due to the fact that $\frac{n}{\min_j|P_j|}<n$ and $\min_j\mathcal{G}_{P_j,t}\geq min_i G_{i,t}$, and equation~\ref{eq:multi_19.5} is due to equation~\ref{eq:mult_p}.

Combining equation~\ref{eq:mult_regret}, equation~\ref{eq:mult_p}, equation~\ref{eq:mult_18}, and equation~\ref{eq:multi_19.final} and setting $\delta'=\min(\frac{1}{3nT},\frac{1}{T^{1/3}})$ we get Theorem~\ref{thm:mult_group_bound}.
\end{proof}

\section{Additional Experiments}\label{app:experiments}
Additionally to the experiments found in Section~\ref{sec:synth}, we ran the following experiments and found no interesting effects:
\begin{enumerate}[(a)]
    \item Varying the range in which coefficients are chosen (between [0,c]) while setting the total budget $T=1000$, the number of arms $n=10$, the error mean $\mu=10$, the number of sensitive arms equal to 5, and the context dimension $d=2$ (Figures~\ref{fig:pulls_c} and \ref{fig:regret_c}).
    \item Varying the context dimension while setting the total budget $T=1000$, the number of arms $n=10$, the error mean $\mu=10$, and the number of sensitive arms equal to 5 (Figures~\ref{fig:pulls_context} and \ref{fig:regret_context_size}).
    \item Varying probability $\delta$ while setting the total budget $T=1000$, the number of arms $n=10$, the error mean $\mu=10$, the number of sensitive arms equal to 5, and the context dimension $d=2$ (Figures~\ref{fig:pulls_delta} and \ref{fig:regret_delta}).
\end{enumerate}

\begin{figure*}
\centering
\begin{subfigure}[t]{0.30\textwidth}
   \includegraphics[width=1\linewidth]{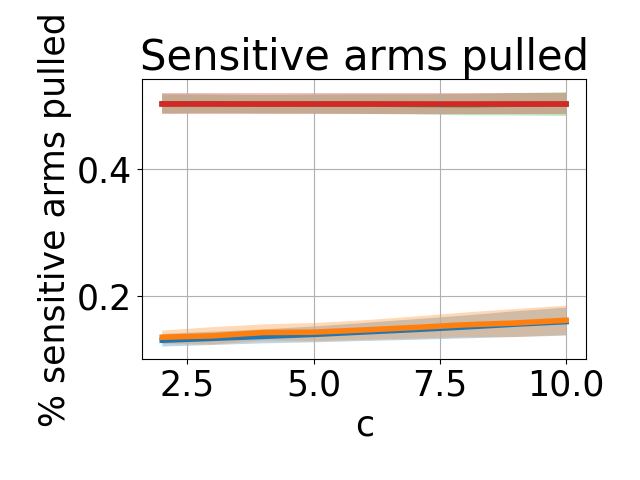}
   \caption{$n=10$, $\mu=10$, number of sensitive arms = 5}
   \label{fig:pulls_c} 
\end{subfigure}
\quad
\begin{subfigure}[t]{0.30\textwidth}
   \includegraphics[width=1\linewidth]{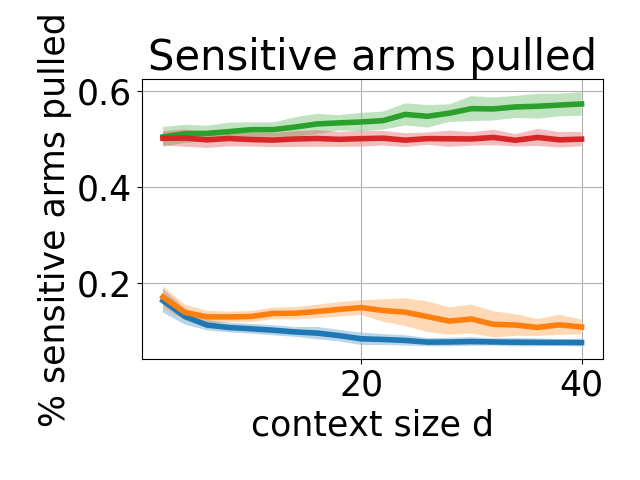}
   \caption{$T=1000$, $\mu=10$, number of sensitive arms = 5}
   \label{fig:pulls_context}
\end{subfigure}

\begin{subfigure}[t]{0.30\textwidth}
   \includegraphics[width=1\linewidth]{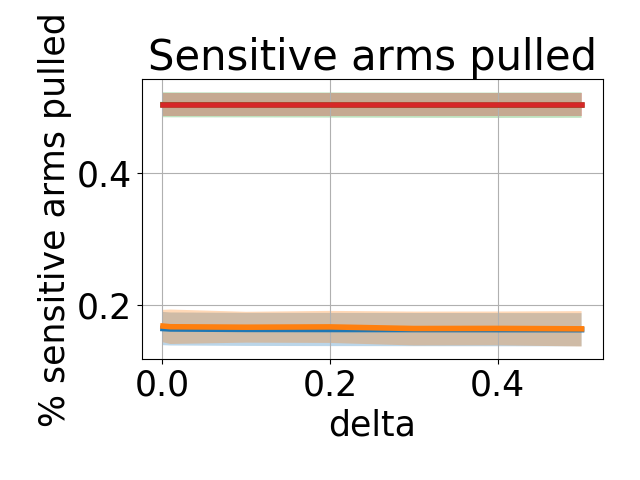}
   \caption{$n=10$, $T=1000$, number of sensitive arms = 5}
   \label{fig:pulls_delta}
\end{subfigure}
\quad
\begin{subfigure}[b]{0.30\textwidth}
   \includegraphics[width=1\linewidth]{graphs/pulls_legend.png}
   \caption{legend}
\end{subfigure}

\caption{Percentage of total arm pulls that were pulled using sensitive arms.}
\end{figure*}

\begin{figure*}
\vspace{-10pt}
\centering
\begin{subfigure}[t]{0.30\textwidth}
   \includegraphics[width=1\linewidth]{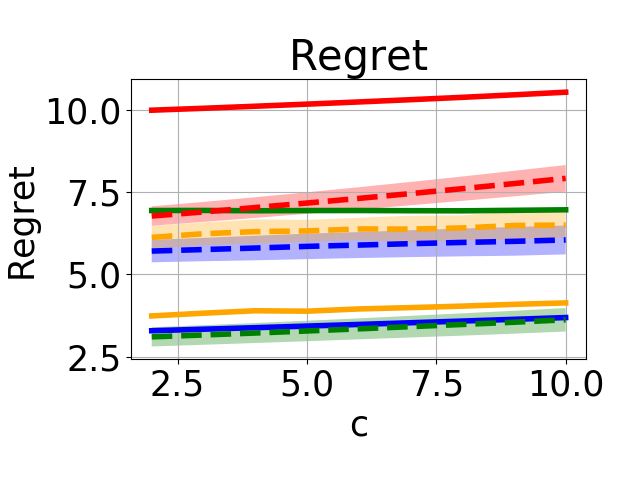}
   \caption{$n=10$, $\mu=10$, number of sensitive arms = 5}
   \label{fig:regret_c} 
\end{subfigure}
\quad
\begin{subfigure}[t]{0.30\textwidth}
   \includegraphics[width=1\linewidth]{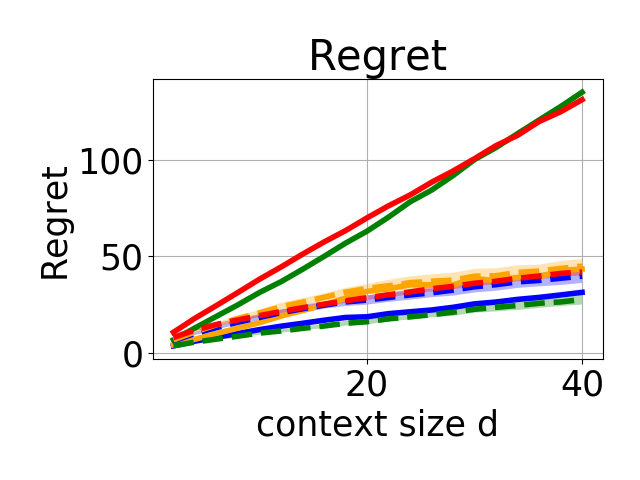}
   \caption{$T=1000$, $\mu=10$, number of sensitive arms = 5}
   \label{fig:regret_context_size}
\end{subfigure}

\begin{subfigure}[t]{0.30\textwidth}
   \includegraphics[width=1\linewidth]{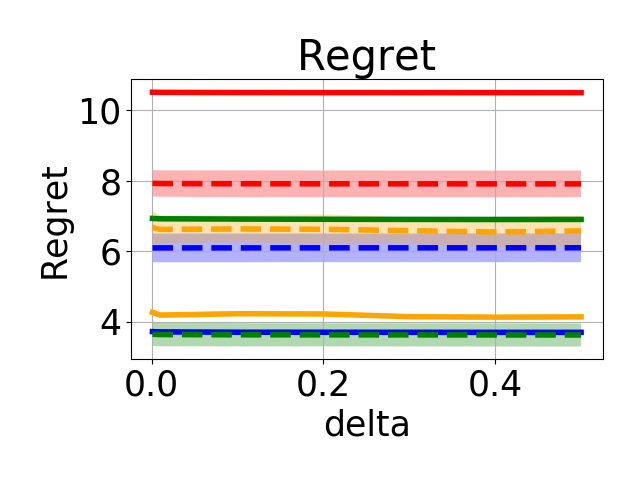}
   \caption{$n=10$, $T=1000$, number of sensitive arms = 5}
   \label{fig:regret_delta}
\end{subfigure}
\quad
\begin{subfigure}[t]{0.30\textwidth}
   \includegraphics[width=1\linewidth]{graphs/regret_legend.png}
   \caption{legend}
\end{subfigure}

\caption{Regret for synthetic experiments. The solid lines are regret given the rewards received from pulling the arms (including the group bias). The dashed lines is the true regret (without the group bias).}
\end{figure*}

\end{document}